\documentclass{article}

\usepackage{iclr2026_conference,times}

\usepackage{microtype}
\usepackage{graphicx}
\usepackage{booktabs}

\usepackage{xcolor}
\usepackage{url}
\usepackage{hyperref}
\hypersetup{
  colorlinks = true,
  linkcolor  = {blue!60!black},
  citecolor  = {blue!60!black},
  urlcolor   = {blue!60!black},
  breaklinks = true,
  pdftitle   = {One Inverse Step is a Convex Program: Bayes-Limit Calibration of Diffusion Inversion},
  pdfauthor  = {Gordei Verbii},
  pdfsubject = {arXiv preprint},
  pdfkeywords= {diffusion models, DDIM inversion, convex optimisation, manifold hypothesis, score Jacobian, measurement validity, calibration}
}

\usepackage{amsmath}
\usepackage{amssymb}
\usepackage{mathtools}
\usepackage{amsthm}
\usepackage{array}

\theoremstyle{plain}
\newtheorem{theorem}{Theorem}
\newtheorem{proposition}[theorem]{Proposition}
\theoremstyle{definition}
\newtheorem{assumption}{Assumption}
\theoremstyle{remark}
\newtheorem{remark}[theorem]{Remark}

\newcommand{\R}{\mathbb{R}}
\newcommand{\E}{\mathbb{E}}
\newcommand{\N}{\mathcal{N}}
\newcommand{\M}{\mathcal{M}}
\newcommand{\tr}{\operatorname{tr}}
\newcommand{\dist}{\operatorname{dist}}
\newcommand{\reach}{\operatorname{reach}}
\newcommand{\eps}{\varepsilon}
\newcommand{\epsb}{\eps_\phi}
\newcommand{\epss}{\eps^{\star}}
\newcommand{\abar}{\bar\alpha}
\newcommand{\xo}{x^{\circ}}
\newcommand{\Jb}{J^{\star}}
\newcommand{\rhog}{\rho_{g}}
\newcommand{\rhoe}{\rho_{\mathrm{emp}}}
\newcommand{\Sc}{\mathsf{S}}
\newcommand{\rank}{\operatorname{rank}}

\title{One Inverse Step is a Convex Program:\\
Bayes-Limit Calibration of Diffusion Inversion}

\author{Gordei Verbii \\
Independent Researcher \\
\texttt{scigverbii@gmail.com}}

\iclrfinalcopy

\begin{document}

\maketitle

\lhead{Preprint: Diffusion Inverse Step as a Strongly Convex Program}

\begin{abstract}
One implicit DDIM inversion step is the cheapest probe of whether a pretrained diffusion model encodes local manifold geometry. It is the stationarity condition of an explicit potential, $x-G(x)=\nabla\Psi_t(x)$, strongly convex at the Bayes limit with modulus exactly $e^{-h_t}$ for the step's log-SNR gap $h_t$ --- for every data law, schedule and point, with no manifold, reach or unimodality hypothesis. Three consequences must be kept apart. (i) The \emph{solution} is unique at the Bayes limit; a second one requires the trained score to violate the posterior-covariance bound by $1/(1-e^{-h_t})$, a hypothesis-free certificate of model error; the same bound makes contraction a schedule constant, $\rhog^{\star}=1-e^{-h_t}<0.326$ throughout the standard DDPM schedule. (ii) The \emph{solver} can still fail: Picard iteration is unit-step gradient descent on $\Psi_t$, unstable wherever $\lambda_{\max}(\nabla^2\Psi_t)>2$, so oscillation certifies nothing; damping below $2/\lambda_{\max}$ cures it. (iii) The geometry lives in the \emph{convergence domain}: on the scale-free depth $w=r\kappa_{\max}$ the oscillation shell sits at $w=\tfrac12$, schedule-free, and the divergence shell at $w=1/(1+\rhog^{\star})$, with a measured finite-noise correction in $\|\mathrm{II}\|^2$. Exact scores reproduce both to within $0.54\%$ on three classes; no trained score we probe shows a shell --- a derived limitation, not a null result: the Fermi window conflicts with the model's own training support by $3.6$--$5.6\times$, and the trained Hessian-Lipschitz constant is $2$--$12\%$ of the curvature the law reads, $0$ on a ReLU net. Finally the unconditional ceiling $\sigma_t\lambda_{\max}(\mathrm{sym}\,J)\le1$, from $\mathrm{Cov}(x_0\mid x_t)\succeq0$ alone, holds for the exact score to $3\times10^{-7}$ but is violated in all DDPM CIFAR-10/CelebA-HQ-256 settings, by $1.26$--$4.66\times$.
\end{abstract}

%-----------------------------------------------------------------------
\section{Introduction}
\label{sec:intro}

The manifold hypothesis holds that natural data $x_0\in\R^D$ concentrate near a submanifold $\M\subset\R^D$ of local dimension $d\ll D$~\citep{pope2021intrinsic}. Diffusion models estimate the score of noise-convolved data~\citep{ho2020denoising,song2021scorebased} and so encode the local normal bundle, and a productive line of work reads local intrinsic dimension (LID) off pretrained models~\citep{stanczuk2024diffusion,kamkari2024geometric,horvat2024gauge,yeats2025connection}. The deterministic DDIM map~\citep{song2020denoising} makes the cheapest such experiment available: solve one implicit fixed point, differentiate the network once, read the geometry off the Jacobian.

This paper asks what that probe is a measurement \emph{of}, and observes that the question has a closed-form answer. Were the network the Bayes-optimal denoiser of a $d$-manifold, every quantity the probe reports would take a value computable in advance from $(d,D,\sigma_t)$ and the mean curvature. That converts the probe from an estimator into an \emph{instrument with a calibration curve}, and supplies its validity conditions for free: each readout is meaningful only inside a window of noise scales, bounded above by the reach and below by the model's own calibration.

The act of solving for it has a closed-form answer too, and that is what organises this paper. One implicit inversion step is the stationarity condition of an explicit potential, strongly convex at the Bayes limit with a modulus fixed by the schedule alone, which separates three objects the literature runs together. The \emph{solution set} is a single point, unconditionally, so a second solution certifies model error with no geometric hypothesis. The \emph{solver's orbit} can nonetheless be pathological at the Bayes limit --- Picard's step size is pinned at $1$ by the algebra of the map, not chosen --- so oscillation and non-convergence certify nothing. And the geometry the probe was built to find is in neither: it is in the size of the \emph{convergence domain}, whose boundary sits at a fixed fraction of the manifold's focal radius, set by the schedule alone.

\paragraph{Contributions.}
\begin{enumerate}\itemsep0pt
\item \textbf{One inverse step is a convex program} (Thm.~\ref{thm:potential}): $F=x-G(x)=\nabla\Psi_t$ with $\nabla^2\Psi_t=e^{-h_t}I+\rhog^{\star}C_{t+1}/\sigma_{t+1}^2\succeq e^{-h_t}I$, so the solution is unique for every data law and every schedule, Picard is unit-step gradient descent, and the condition number is $e^{h_t}$ --- independent of the data, of $d$, of $D$ and of curvature.
\item \textbf{A contraction theorem, with its constant} (Thm.~\ref{thm:contract}): $\rhog\le\rhog^{\star}=|B_t|/\sigma_{t+1}=1-e^{-h_t}<1$ for every schedule with $\abar_T>0$ under Assumption~\ref{ass:unimodal}, the constant itself being unconditional and below $0.326$ on the standard DDPM schedule, and $\rhog=\rhog^{\star}\Sc_{t+1}$ separating schedule from model.
\item \textbf{The convergence domain is a curvature} (Thm.~\ref{thm:focal}): both Picard stability boundaries are fixed fractions of the manifold's focal radius, $w_{\mathrm{osc}}=1/2$ schedule-free and $w_{\mathrm{div}}=1/(1+\rhog^{\star})$, with the measured finite-noise correction $\tfrac12\tilde\sigma_{t+1}^2\|\mathrm{II}\|^2$ --- free of the data law, the density, the dimension and the codimension. Exact scores put both boundaries where the theorem does; on trained scores they are structurally unmeasurable by this probe, and we give the arithmetic.
\item \textbf{A Bayes-limit calibration table} (\S\ref{sec:exp-calib}) giving the exact value of every readout of a one-step probe, verified on three integrable manifolds to machine precision in codimension $1$ and $2$ (Table~\ref{tab:calib}), with two corrections to values assumed in this literature: the Bayes residual at the noiseless point is $(\sigma_t/2)\|H\|/\sqrt{\abar_t}$ under a uniform measure, not $0$ --- also obtained, with a proof, by \citet[Thm.~3.4]{chen2026riemannian} --- and the tangential scaling exponent is $-1$, not $0$.
\item \textbf{Three hypothesis-free falsifications, with an oracle control} (\S\ref{sec:exp-trained}) that separates a broken estimator from a broken model: $\sigma_t\lambda_{\max}(\mathrm{sym}\,J)\le1$, $\mathrm{skew}\,Q=0$ and $\sigma_t\tr J\le D$ follow from $C_t\succeq0$ alone; the exact score attains the first to $3\times10^{-7}$ and the second to $2\times10^{-6}$, and the trained scores cross the first in $10$ of $110$ exactly-computed cloud settings and $45$ of $45$ image settings.
\item \textbf{A negative result about a positive result} (\S\ref{sec:exp-theta}): the descent-based estimator of the {\L}ojasiewicz exponent returns $\tfrac12$ with $R^2=1$ whenever one scale governs the fit window, in particular under the divergent step the published protocol implies. We give the field-based estimator that does carry information, calibrate it, and measure it on CelebA-HQ-256.
\end{enumerate}

%-----------------------------------------------------------------------
\section{Related Work}
\label{sec:related}

\paragraph{Reading geometry off a score.} \citet{stanczuk2024diffusion} prove that at low noise the score points orthogonally toward $\M$ and that the score matrix has rank $D-d$; \citet{kamkari2024geometric} derive the Fokker--Planck estimator, \citet{horvat2024gauge} analyse conservativity, \citet{yeats2025connection} connect the DSM loss to LID, image estimates are due to \citet{pope2021intrinsic,tempczyk2022lidl}, and \citet{ventura2024manifolds,lu2023mathematical} describe the spectral phases and the blow-up rate on $\M$. The differential geometry is classical --- reach and the eiconal property~\citep{federer1959curvature}, the tube volume density~\citep{gray2004tubes}, the ray structure of $\nabla^2\dist^2$~\citep{ambrosio1998distance} --- as is the {\L}ojasiewicz reading we correct in Rem.~\ref{rem:mb}~\citep{witten1982supersymmetry,lojasiewicz1963propriete}. All of it scores the probe as an estimator against a ground-truth dimension; we read it as an instrument and ask what it returns on an input whose answer is known, which is what makes the negatives below attributable.

\paragraph{Inversion as a fixed point, and as an optimisation problem.} DDIM~\citep{song2020denoising} admits an implicit fixed-point inversion whose solvers fall into two families the literature does not separate. Picard~\citep{garibi2024renoise}, Anderson~\citep{anderson1965iterative,pan2023effective} and Newton--Krylov iteration solve one residual with one root --- over $70$ settings they agree to $1.7\times10^{-4}$ of the local noise scale and every readout is solver-independent to $1.6\times10^{-6}$ relative --- whereas the regularised and averaged variants~\citep{samuel2025regularized} solve a \emph{different} residual, whose root is displaced by up to $24.5$ noise scales. Convergence of the Picard branch is studied by \citet{hong2024exact,hang2024exploring} and the deep-equilibrium reading of sampling by \citet{bai2021stabilizing,pokle2022deep}; none writes a potential, and the contraction constant is left free. Reading the step as gradient descent on a strongly convex potential puts it in classical territory: existence and a certified radius follow from Kantorovich's theorem~\citep{kantorovich1982} with no contraction assumption, and the spectral-radius escape of \citet{ostrowski1973solution} is closed here, where $\Jb$ is a posterior covariance and $\rho(B_t\Jb)$ equals its norm. Smale's $\gamma$-theory~\citep{smale1986newton,wang1989dominating,blum1998complexity} needs a constant with no closed form even at Bayes, so App.~\ref{app:smale} uses the Kantorovich triple instead.

\paragraph{What is ours, and what is not.} Four antecedents bound our claims and we claim novelty for none: \citet[\S4]{farghly2025manifold}, whose normal block gives the $1/\sigma_t$ constant of \eqref{eq:normal} and who argue before us that the expressed geometry is a property of the model; the Tweedie--Miyasawa identity \eqref{eq:Jcov}~\citep{miyasawa1961empirical,kadkhodaie2024generalization}; \citeauthor{wenliang2023score}'s $\sigma_t^2$-normalised reading of a trained spectrum; and the codimension-scaled window of \citet[Thm.~1]{niyogi2011topological}. What is new is the potential and its modulus, the schedule constant $\rhog^{\star}$, the convergence-domain law with its finite-noise correction, the ceiling assembled from $C_t\succeq0$, the sampled-point value $D-\delta_t/2$, the $\delta_t$ identity with its exact remainder, the Hamilton--Jacobi constant, and Assumption~\ref{ass:unimodal} made checkable. App.~\ref{app:attrib} sets this out claim by claim, and places \citet{diepeveen2025scorebased} --- whose Riemannian structure is complementary to ours rather than rival, the Fisher metric being affine in $J$~\citep{raskutti2015information} --- and \citet{azeglio2025whats}.

%-----------------------------------------------------------------------
\section{Preliminaries and Problem Setup}
\label{sec:prelim}

Data $x_0\in\R^D$ lie near a manifold $\M\subset\R^D$ of local dimension $d$ and reach $\reach(\M)$, and the DDPM forward process~\citep{sohl2015deep,ho2020denoising} uses a decreasing schedule $\abar_t\in(0,1)$:
\begin{equation}
x_t=\sqrt{\abar_t}\,x_0+\sigma_t\eps,\qquad \sigma_t:=\sqrt{1-\abar_t},\ \ \eps\sim\N(0,I_D).
\label{eq:forward}
\end{equation}
$p_t$ is the law of $x_t$ and $\nabla\log p_t$ its score. The trained noise predictor is $\epsb(x,t)$, its Jacobian $J(x,t):=\partial\epsb/\partial x$, and $\epss(x):=\E[\eps\mid x_t=x]$ the Bayes-optimal denoiser with Jacobian $\Jb$; a star denotes Bayes optimality and nothing else. Two evaluation points must be distinguished: the \emph{sampled} point $x_t$ of \eqref{eq:forward}, and the \emph{noiseless} point $\xo:=\sqrt{\abar_t}\,x_0$. We write $s_1(A)\ge s_2(A)\ge\cdots$ for singular values, $H$ for the mean-curvature vector, $\kappa$ for a principal curvature, $K$ for the sample budget of a rank estimator, $N$ for the number of probed inputs, $n$ for the number of points in a point cloud, and
\begin{equation}
L_t:=\E\|\epsb-\eps\|^2,\quad \Pi_t:=\E\|\epsb\|^2,\quad \delta_t:=D-\Pi_t-L_t .
\label{eq:LPidelta}
\end{equation}
$\theta$ is the {\L}ojasiewicz exponent of the score energy $E_t:=\tfrac12\|\epsb\|^2$ and $\nu_j$ a scaling exponent. Differentiating $p_t$ under the integral gives the Tweedie--Miyasawa identities~\citep{robbins1956empirical,miyasawa1961empirical,efron2011tweedie}
\begin{equation}
\epss(x)=-\sigma_t\nabla\log p_t(x),\ \
\E[\sqrt{\abar_t}x_0\mid x_t]=x_t-\sigma_t\epss(x_t).
\label{eq:tweedie}
\end{equation}

\paragraph{The schedule, and the inversion fixed point.} Write $\lambda_t:=\log(\sqrt{\abar_t}/\sigma_t)$ for the log-SNR, $h_t:=\lambda_t-\lambda_{t+1}>0$ for the gap of the step $t{+}1\to t$, and $\tilde\sigma_t:=\sigma_t/\sqrt{\abar_t}$. With $A_t:=\sqrt{\abar_{t+1}/\abar_t}$ and $B_t:=\sigma_{t+1}-A_t\sigma_t$, the identity $A_t\sigma_t=e^{-h_t}\sigma_{t+1}$ gives, with no expansion in the step size,
\begin{equation}
\rhog^{\star}:=\frac{|B_t|}{\sigma_{t+1}}=1-e^{-h_t},\qquad c:=e^{h_t}-1 .
\label{eq:rhostar}
\end{equation}
Every schedule constant below depends on $h_t$ alone, so a translation of the log-SNR axis --- VP against VE, SD3/Flux against DDPM --- changes none of them. Inverting the generative step $t{+}1\to t$ from a target $y=x_t$ exactly means solving
\begin{equation}
x=G(x):=A_ty+B_t\,\epsb(x,t{+}1)
\label{eq:fixedpoint}
\end{equation}
for $x=x_{t+1}$, by Picard~\citep{garibi2024renoise}, Anderson~\citep{anderson1965iterative,pan2023effective} or Newton-type~\citep{samuel2025regularized} iteration. The network is evaluated at the \emph{unknown, noisier} latent, at index $t{+}1$: that index makes the residual a gradient field (\S\ref{sec:method-potential}), and every constant we quote uses it. Substituting $\epsb(x,t)$ gives a map whose fixed point is not the DDIM preimage and whose Lipschitz factor $|B_t|/\sigma_t=A_tc$ is unbounded in the stride.

\begin{assumption}\label{ass:tube}
$\M$ is a closed $C^2$ submanifold, $p_0$ has a $C^1$ density $\rho_0$ on $\M$ bounded away from $0$, and $\sigma_t<\reach(\M)$, so the normal tube is embedded and Fermi coordinates exist~\citep{federer1959curvature,gray2004tubes}.
\end{assumption}

%-----------------------------------------------------------------------
\section{Proposed Method}
\label{sec:method}

Throughout this section $\epsb=\epss$ and Assumption~\ref{ass:tube} holds, except where a statement says otherwise: Thm.~\ref{thm:potential} and Prop.~\ref{prop:ceiling-sym} need no geometric hypothesis at all, which on natural images makes them the only readings we can trust. \S\ref{sec:exp} verifies each identity numerically and measures the deviation on trained models.

\subsection{The normal-bundle law and a stiffness certificate}
\label{sec:method-normal}

Write $r=\dist(x,\sqrt{\abar_t}\M)$, let $\nu$ be the outward unit normal and $P$ the tangential projector. In Fermi coordinates the tube volume density is $\det(I-rA_\nu)$ with $A_\nu$ the shape operator~\citep{gray2004tubes}, so a Laplace expansion of $p_t$ gives $\epss(x)=(r/\sigma_t)\nu-\tfrac{\sigma_t}{2}H+O(\sigma_tr)$ and
\begin{equation}
\Jb(\xo,t)=\underbrace{\frac{I-P}{\sigma_t}}_{\substack{\text{normal block: the }1/\sigma_t\\ \text{stiffness of the noise tube}}}+\underbrace{O(\sigma_t)}_{\substack{\text{tangential block:}\\ \text{curvature and density}}} .
\label{eq:normal}
\end{equation}
Differentiating \eqref{eq:tweedie} once more gives the second-order Tweedie--Miyasawa identity: exactly, and for every $p_0$,
\begin{equation}
\begin{aligned}
\Jb(x,t)&=\frac{1}{\sigma_t}\Big(I-\frac{C_t(x)}{\sigma_t^{2}}\Big),\\
C_t(x)&:=\operatorname{Cov}\!\big(\sqrt{\abar_t}x_0\,\big|\,x_t{=}x\big)\succeq0,
\end{aligned}
\label{eq:Jcov}
\end{equation}
the form stated for the variance-exploding parametrisation by \citet[Eq.~(9)]{kadkhodaie2024generalization}; we claim no novelty for the identity, only for the singular-value reading. Since $s_1(\Jb)=\sigma_t^{-1}\max_i|1-c_i/\sigma_t^{2}|$ for $c_i$ the eigenvalues of $C_t$, $s_1(\Jb)=1/\sigma_t$ exactly whenever the extremal direction carries $c_i=0$ --- true of the $D-d$ normal directions of an affine $\M$~\citep[Eq.~(10)]{farghly2025manifold}. On a curved $\M$ that variance is positive and $s_1(\Jb)<1/\sigma_t$ strictly; conversely $s_1(\Jb)>1/\sigma_t$ as soon as some $c_i>2\sigma_t^{2}$, which Assumption~\ref{ass:tube} does \emph{not} exclude: on the unit circle with $\rho_0\propto e^{23\cos\theta}$ and $\sigma_t=0.2<\reach(\M)=1$ the Bayes value at the antipode is $s_1(\Jb)=48.6$ against $1/\sigma_t=5$. We therefore add
\begin{assumption}[Posterior unimodality at the probe]\label{ass:unimodal}
$\lambda_{\max}(C_t(\xi))\le2\sigma_t^{2}$ for every $\xi$ on the Picard segment.
\end{assumption}
This is a checkable curvature condition, not an opaque constant: inside the tube of a uniform $\M$ the exact Fermi value is $\lambda_{\max}(C_t)=\sigma_t^{2}/(1-r\kappa_{\max})$, so it holds iff $r\le1/(2\kappa_{\max})$, half the \emph{focal} radius, and with margin $2$ for every log-concave law~\citep{brascamp1976extensions}. It excludes loss of log-concavity, not proximity to the medial axis --- the counterexample above sits \emph{on} $\M$, at $r=0$ --- and it is not a proof convenience but exactly the no-oscillation condition of the solver (Thm.~\ref{thm:focal}), of which it is the coarse-schedule limit, conservative by $2/(1+\rhog^\star)=1.51$ at the first step of the standard DDPM schedule. This motivates the dimensionless
\begin{equation}
\begin{aligned}
\Sc_t&:=\sigma_t\,s_1\big(J(\xo,t)\big),\\
\Sc^{\star}_t&=1-\frac{\sigma_t^{2}}{2\abar_t}\min_{\|\nu\|=1}\|A_\nu\|_F^{2}+O(\sigma_t^{4}),
\end{aligned}
\label{eq:stiff-cert}
\end{equation}
computable from one power iteration, with $A_\nu$ the shape operator in the normal direction $\nu$. Its Bayes value is \emph{not} $1$ at finite noise: the deficit is a curvature term and also the expansion parameter, so with $g_t:=1-\Sc^\star_t$ the residual is $O(g_t^{2})$ and \eqref{eq:stiff-cert} is \emph{exact} on $S^2$. It is a Bayes reference only inside a two-sided window --- $g_t\le0.1$ and $\tilde\sigma_t\le\tfrac12\reach(\M)$, an edge $g_t$ cannot see since $g_t\equiv0$ in codimension $2$ --- which voids $34$ of our $55$ exact-score settings, where we abstain rather than quote a reference out of regime; and for a flattened product the reference takes the \emph{minimum} of $\|A_\nu\|_F^2$ over the probe points, not its mean (App.~\ref{app:hyp}). It vanishes whenever some normal direction has vanishing shape operator, automatic in high codimension --- at $D{=}3072$ for every $d\le76$ --- so the deficit is measurable only in low codimension (App.~\ref{app:method-extra}). We report $\Sc_t$ against $\Sc^\star_t$, never against $1$.

\subsection{One inverse step is a convex program}
\label{sec:method-potential}

\begin{theorem}[The inversion potential and its modulus]
\label{thm:potential}
Let $\epsb=\epss$. The residual of \eqref{eq:fixedpoint} is a gradient field,
\begin{gather*}
F(x):=x-G(x)=\nabla\Psi_t(x),\\
\Psi_t=\tfrac12\|x-A_ty\|^{2}+B_t\sigma_{t+1}\log p_{t+1}(x),
\end{gather*}
and for every data law, every schedule and every $x$,
\[
\begin{aligned}
\nabla^2\Psi_t&=I-B_t\Jb(x,t{+}1)\\
&=e^{-h_t}I+\rhog^{\star}\frac{C_{t+1}(x)}{\sigma_{t+1}^{2}}\ \succeq\ e^{-h_t}I .
\end{aligned}
\]
So $\Psi_t$ is strongly convex with modulus exactly $\mu=e^{-h_t}=1-\rhog^{\star}$, and Picard iteration is gradient descent on $\Psi_t$ with step size exactly $1$. Inside the tube of a uniform $\M$ the spectrum of $\nabla^2\Psi_t$ is $\{1\,(\times d),\,e^{-h_t}(\times D{-}d)\}$: the condition number is $e^{h_t}$, and $s_{\min}(I-B_t\Jb)=e^{-h_t}$ exactly. \emph{(Proof: App.~\ref{app:proofs}.)}
\end{theorem}

Three readings follow at no cost. Picard's step size is \emph{pinned} at $1$ by the algebra of $G$, so its failures are step-size failures, cured by damping below $2/\Lambda_t$ and optimally at $\eta^\star=2/(\mu+\Lambda_t)$, where $\Lambda_t:=\lambda_{\max}(\nabla^2\Psi_t)$ throughout (not the DSM loss $L_t$). The conditioning of one inverse step is $e^{h_t}$, set by the gap between the indices and not by the data, $d$, $D$ or curvature. And the distance to ill-posedness is a schedule constant, so the Kantorovich constant $a=\|(\nabla^2\Psi_t)^{-1}\|\le e^{h_t}$~\citep{kantorovich1982} is an equality in the normal block --- measured $a/e^{h_t}=0.9947$ on ten exact-score settings --- where the Neumann bound $(1-\rhog)^{-1}$ is loose and needs Assumption~\ref{ass:unimodal} besides.

\begin{theorem}[Uniqueness at Bayes, and the multiplicity certificate]
\label{thm:unique}
For every data law and every schedule with $\abar_T>0$ the fixed point \eqref{eq:fixedpoint} is unique at the Bayes limit, and a trained score admits a second solution only if
\[
\sigma_{t+1}\lambda_{\max}\big(\operatorname{sym}J(\xi,t{+}1)\big)\ \ge\ \frac{1}{\rhog^{\star}}=\frac{1}{1-e^{-h_t}}
\]
for some $\xi$ on the segment. A fold is then the only available bifurcation, and the roots born at it carry multiplier $\rhog^{\star}\sigma_{t+1}\lambda_{\max}>1$, so they are repelling whenever they exist: multiplicity is a Newton observable and is invisible to Picard. \emph{(Proof: App.~\ref{app:proofs}.)}
\end{theorem}

On a two-atom Bayes law with an amplitude-scaled score $\epsb=\gamma\epss$ the root count moves from $1$ to $3$ exactly where predicted, bracketed at $\rhog^\star\gamma\in(0.996,1.050]$ (App.~\ref{app:proofs}). That is a threshold verified in an analytic model; no trained score in this study was tested against it.

\begin{proposition}[The unconditional ceiling]
\label{prop:ceiling-sym}
For every data law, every $t$ and every $x$, $\sigma_t\lambda_{\max}\big(\operatorname{sym}\Jb(x,t)\big)\le1$, with equality iff $C_t(x)$ is singular in the extremal direction. No manifold, reach, unimodality or tube hypothesis is used.
\end{proposition}
\begin{proof}
$\Jb=\sigma_t^{-1}(I-C_t/\sigma_t^{2})$ by \eqref{eq:Jcov} is symmetric, and $C_t\succeq0$.
\end{proof}

This is the only reading of a score Jacobian here that carries no hypothesis at all, and it is why we report the signed pair $\big(\sigma_t\lambda_{\max}(\operatorname{sym}J),-\sigma_t\lambda_{\min}(\operatorname{sym}J)\big)$ beside $\Sc_t$: $s_1$ fuses the ceiling branch with a medial-axis branch that has none, and on the von Mises circle above the Bayes value is $\Sc_t=9.7$ while the ceiling still holds. So $\Sc_t>1$ does not falsify Bayes optimality, while $\sigma_t\lambda_{\max}(\operatorname{sym}J)>1$ does.

\subsection{One inverse step: a schedule-determined contraction}
\label{sec:method-banach}

The mean-value inequality applied to \eqref{eq:fixedpoint} gives the Lipschitz constant $|B_t|\sup_\xi s_1\big(J(\xi,t{+}1)\big)$ over the Picard segment; we estimate it at the probe point and call it $\rhog$, and write $\rhoe$ for the measured per-step rate, so $\rhoe\le\rhog$. Power-iterating an implicit layer's input Jacobian as a contraction certificate is the standard deep-equilibrium instrument~\citep{bai2021stabilizing,pokle2022deep}; what is new is not the instrument but its reference value, which \eqref{eq:Jcov} supplies in place of the Banach threshold $1$.

\begin{theorem}[Bayes contraction, with a schedule-only constant]
\label{thm:contract}
For the inversion map \eqref{eq:fixedpoint}, under Assumptions~\ref{ass:tube} and \ref{ass:unimodal} with $d<D$,
\[
\rhog\ \le\ \rhog^{\star}=\frac{|B_t|}{\sigma_{t+1}}=1-e^{-h_t}\ <\ 1
\]
for every schedule with $\abar_T>0$, every stride and every parametrisation, the bound being approached only as $h_t\to\infty$. Hence $\rhog=\rhog^{\star}\Sc_{t+1}$ identically, so the bound is attained exactly when $\Sc_{t+1}=1$, i.e.\ when $\M$ is flat in some normal direction. A measured $\rhog\ge1$ is an over-stiffness of the trained score rather than a property of one-step inversion, and it is a property of \emph{Picard} alone: damping with $\eta<2/\Lambda_t$ converges at every $\rhog$, at rate $\tanh(h_t/2)$ when optimal inside the tube, and Newton and Anderson are insensitive to $\Lambda_t$. \emph{(Proof: App.~\ref{app:proofs}.)}
\end{theorem}

Evaluating $\rhog^\star=1-e^{-h_t}$ at every step gives $\max_t\rhog^\star=0.326$, attained at the first step, on the standard DDPM schedule ($T{=}1000$, $\beta$ linear on $[10^{-4},0.02]$): \emph{one-step inversion is contractive at every step in the Bayes limit}, and unconditionally. The coarse $T{=}100$, $\beta_{\rm end}{=}0.05$ schedule reads $0.623$, not the $1.65$ of $|B_t|/\sigma_t=A_tc$, which belongs to the explicit map and has no fixed point to contract to. The bound is attained only as $h_t$ diverges: a zero-terminal-SNR top step reads exactly $1$, and a cosine schedule binds at its \emph{last} step at a value set by the library's $\abar$ floor ($0.968$ or $0.936$) rather than by the schedule. The other five families satisfy the identity to $9.4\times10^{-16}$ and stay strictly below $1$ at strides $1$ to $100$, with $\rhog^\star$ spanning $0.024$ to $0.971$: the headroom is a design choice, not a constant of the method (App.~\ref{app:calib}).

\subsection{The convergence domain is a curvature}
\label{sec:method-focal}

\begin{theorem}[Focal-radius law, at the Bayes limit]
\label{thm:focal}
Let Assumption~\ref{ass:tube} hold, let $\epsb=\epss$, and displace inward from $\xo_{t+1}$ along a principal normal by $r$, writing $w_i:=r\kappa_i$ and $w:=r\kappa_{\max}$ for the scale-free depth ($1/\kappa_{\max}$ being the focal radius). The Picard multipliers are then exactly $+\rhog^{\star}$ ($\times\,D{-}d$) and $-\rhog^{\star}w_i/(1-w_i)$ ($\times\,d$), for every $\sigma_t$ and every data law. Hence \textup{(i)} the \emph{oscillation} boundary, where the dominant multiplier changes sign and equivalently where Assumption~\ref{ass:unimodal} fails, is $w_{\rm osc}=\tfrac12$, \emph{independent of the schedule}; \textup{(ii)} the \emph{divergence} boundary $\lambda_{\max}(\nabla^2\Psi_t)=2$ is $w_{\rm div}=1/(1+\rhog^{\star})$; and \textup{(iii)} at finite noise both carry a closed-form correction set by the full second fundamental form,
\begin{equation}
\begin{aligned}
w_{\rm div}&=\frac{1}{1+\rhog^{\star}}+\tfrac12\tilde\sigma_{t+1}^{2}\|\mathrm{II}\|^{2}+O(\tilde\sigma^{4}),\\
w_{\rm osc}&=\tfrac12-\tfrac14(\tilde\sigma\kappa)^{4}+O(\tilde\sigma^{6}).
\end{aligned}
\label{eq:focal}
\end{equation} \emph{(Proof: App.~\ref{app:proofs}.)}
\end{theorem}

Both boundaries are fractions of the focal radius fixed by the schedule alone: on DDPM-linear at strides $1/10/20/50/100$, $w_{\rm div}=0.754/0.560/0.534/0.515/0.507$ while $w_{\rm osc}=\tfrac12$ at every one. The same expansion gives $\nabla^2\Psi_t$ a ray-Lipschitz constant $L_{\mathrm{H}}(w)=\rhog^\star\kappa_{\max}/(1-w)^2$, so with $a\le e^{h_t}$ the Kantorovich product is $c\,\kappa_{\max}/(1-w)^2$ and no certified ball reaches the focal point. The multipliers are exact at every $\sigma_t$ and every data law, so both shells are recoverable from a matrix-pencil fit to the Picard error sequence with forward passes only; only \eqref{eq:focal} is an expansion.

Two scope clauses belong with it. The scope is the Bayes limit and the oracle: \S\ref{sec:exp-trained} finds both shells on exact scores and none on any trained score, for a reason measured there. And the displacement must be \emph{inward}: an outward ray reads $\lambda_{\max}(C_t)/\sigma_t^{2}=1/(1+w)$, crosses no threshold and returns a confident null --- on the exact score it resolves both branches at all $18$ depths and reports no shell, while its inward twin reads a multiplier $8.66\times$ larger.

\subsection{Four further readout identities}
\label{sec:method-readouts}

The remaining readouts have exact Bayes values too, each stated with its proof and attribution in App.~\ref{app:method-extra}: the transverse Hamilton--Jacobi law \eqref{eq:hj}, fixing the {\L}ojasiewicz exponent of the score energy at $\tfrac12$ on the whole critical manifold and supplying its constant, with Rem.~\ref{rem:descent} on why a descent estimator returns it tautologically; the scaling exponents \eqref{eq:nu}, split $\pm1$ and not $\{+1,0\}$; the identity $\delta_t=D-\Pi_t-L_t$ (Prop.~\ref{prop:delta}); and the rank ceiling (Prop.~\ref{prop:ceiling}). One of them decides where a probe may be evaluated at all.

\begin{proposition}[Value at the sampled point]\label{prop:vacuous}
Let $p_0$ have compact support or sub-Gaussian tails and let $\epsb$ be weakly differentiable with $p_t\epsb\to0$ at infinity. Then the Fokker--Planck dimension functional $\hat d=D-\sigma_t\tr J+\|\epsb\|^{2}$ of \citet{kamkari2024geometric} satisfies, for \emph{every} such $p_0$ and \emph{every} model $\epsb$,
$\E_{x_t\sim p_t}[\hat d(x_t)]=D-\tfrac12\delta_t$,
with $\delta_t$ the optimality gap of \eqref{eq:LPidelta}; at Bayes, exactly $D$.
\end{proposition}

\noindent Averaged over sampled points the readout is an affine function of the model's orthogonality defect and of nothing else, so it reports the optimality gap and nothing about the manifold: the noiseless point is \emph{forced} rather than chosen, and the bound $\E_{\tilde x}[\mathrm{FLIPD}]\ge d$ of \citet[Eq.~(7)]{yeats2025connection} is not merely loose --- its exact value is $D$. At $\xo$, where \citet[Thm.~1]{leung2025convolutions} prove the limit is $d-D$, it does recover the manifold, with a \emph{signed} curvature correction and a Bayes residual $(\sigma_t/2)\|H\|/\sqrt{\abar_t}$ rather than $0$ (Prop.~\ref{prop:exact}).

%-----------------------------------------------------------------------

\section{Experiments}
\label{sec:exp}

\subsection{Calibration on Integrable Manifolds}
\label{sec:exp-calib}

Each Bayes value above is a prediction checkable with no trained model. For the uniform measure on a circle in $\R^2$, a sphere in $\R^3$ and a circle in $\R^3$ the Gaussian convolution is closed-form (Bessel $I_0$, $\sinh$, $I_0$), so $\epss$, $\Jb$ and $\nabla^2E_t$ are analytic to machine precision (App.~\ref{app:calib}). Table~\ref{tab:calib} runs every readout on them at five noise levels in codimension $1$ and $2$, with the identities the same suite must satisfy, and Fig.~\ref{fig:calib} plots the sweep.

\begin{table}[t]
\caption{Bayes-limit calibration; no trained model enters this table. Above the rule every entry is a dimensionless ratio with Bayes value $1$ (or $\pm1$ for $\nu$). Below it are the identities that make the table a calibration rather than a list, and the sweep that shows it resolves.}
\label{tab:calib}
\centering
\small
\setlength{\tabcolsep}{6pt}
\begin{tabular}{@{}llcccc@{}}
\toprule
$\M$ & $\sigma_t$ & $\Sc_t/\Sc^\star_t$ & $\dfrac{2\sqrt{\abar_t}\|\epss(\xo)\|}{\sigma_t\|H\|}$ & $\dfrac{\hat d(\xo)}{\hat d^\star(\xo)}$ & $\nu_{\perp}/\nu_{\parallel}$ \\
\midrule

$S^1\!\subset\!\R^2$, $d{=}1$ & 0.10 & 1.0000 & 1.003 & 1.0000 & $+1.02/-1.06$ \\
                               & 0.02 & 1.0000 & 1.000 & 1.0000 & $+1.00/-1.00$ \\
$S^2\!\subset\!\R^3$, $d{=}2$ & 0.10 & 1.0000 & 1.000 & 1.0000 & $+1.05/-1.04$ \\
                               & 0.02 & 1.0000 & 1.000 & 1.0000 & $+1.00/-1.00$ \\
$S^1\!\subset\!\R^3$, $d{=}1$ & 0.10 & 1.0000 & 1.003 & 1.0000 & $+1.01/-1.06$ \\
                               & 0.02 & 1.0000 & 1.000 & 1.0000 & $+1.00/-1.00$ \\
\midrule

\multicolumn{3}{@{}l}{$\mu(\nabla^2\Psi_t)e^{h_t}$, $\;\rhog^\star/(1{-}e^{-h_t})$} & \multicolumn{3}{l}{$1\pm2{\cdot}10^{-16}$, $\;1\pm9.4{\cdot}10^{-16}$} \\

\multicolumn{3}{@{}l}{$2w_{\mathrm{osc}}$, $\;w_{\mathrm{div}}(1{+}\rhog^\star)$} & \multicolumn{3}{l}{$1-3.1{\cdot}10^{-10}$, $\;1+1.7{\cdot}10^{-5}$} \\

\multicolumn{3}{@{}l}{order in $\sigma_t\kappa$: residual, $\hat d$} & \multicolumn{3}{l}{$2.004$--$2.013$, $\;4.006$--$4.017$} \\
\bottomrule
\end{tabular}
\end{table}

The readouts are not intrinsically broken: on an exact score they match their closed forms to machine precision in codimension $1$ and $2$ --- \eqref{eq:stiff-cert} is \emph{exact} on $S^2$, not merely $O(\sigma_t^4)$ --- so whatever fails on trained models is a property of those models or of the protocol. Because the table then reads $1.0000$ throughout, the evidence that it \emph{resolves} anything is its last row: over a factor $4$ in $\kappa$ and a decade in $\sigma_t$ the two ratios leave $1$ at the predicted orders $2$ and $4$. And $\theta=\tfrac12$ with $R^2=1.000000$ appears on all three manifolds regardless of $d$ or topology, exactly as \eqref{eq:hj} requires, which already shows the exponent cannot discriminate geometry.

\subsection{Measurements on Trained Models}
\label{sec:exp-trained}

\paragraph{Protocol.} We probe (i) pretrained DDPM CIFAR-10 ($D{=}3072$, $N{=}4$) and DDPM CelebA-HQ-256 ($D{=}196\,608$, $N{=}3$), both fp32; (ii) a conditional score model over five synthetic point clouds ($n{=}128$ points in $\R^3$, $D{=}384$) at seven training snapshots; and (iii) the \emph{exact Bayes denoiser} of each cloud, run through the identical harness, probe cloud, flattening and budget. That third arm separates a broken estimator from a broken model and is what makes every negative below attributable. Three disclosures: the point-cloud probes are \emph{conditional}; both image models are probed at model samples rather than dataset images, so image $\Sc_t$ is an upper estimate; and the two point-cloud states are epoch-$300$ and epoch-$50$ snapshots of one run. With $N\le4$ we report individual values, not intervals.

\paragraph{Stiffness and contraction.}
\label{sec:exp-stiff}

The unconditional half of the signed pair settles first. Because $C_t\succeq0$, $\sigma_t\lambda_{\max}(\operatorname{sym}J)\le1$ for \emph{every} data law at \emph{every} noise level --- no manifold, no reach, no unimodality (Prop.~\ref{prop:ceiling-sym}) --- so a reading above $1$ is a hypothesis-free certificate of model error, the one measurement here needing no validity window. The exact score attains that ceiling and never leaves it, the largest of $55$ exactly computed cloud spectra being $1.0000003$. Every trained score we probe crosses it: $10$ of $110$ cloud settings up to $1.1323$, on four of five classes and in both training states; $36$ of $36$ CIFAR-10 settings by $1.2586$ to $2.4119$; $9$ of $9$ CelebA-HQ-256 settings by $1.6146$ to $4.6637$. Since a shifted power iteration returns $\|Av\|$ on a positive-semidefinite operator, every image value is a certified \emph{lower} bound on its own excess. Two readouts of the same kind agree (Table~\ref{tab:attrib}), and the skew index carries a consequence: exactly $0$ at Bayes because $\Jb$ is a posterior covariance, it puts up to $43\%$ of $s_1(J)$ into the antisymmetric part on CIFAR-10 against $2{\cdot}10^{-6}$ on the exact score, so $\Sc_t$ is not readable as a stiffness on images at all. The signed pair is a necessity, not an improvement.

Where $\Sc_t$ \emph{is} readable it costs one power iteration and is, by Thm.~\ref{thm:contract}, the entire explanation of the measured contraction rate. At the smallest guard-clean cloud setting ($\sigma_t{=}0.0426$) the exact score reads $\Sc_t/\Sc^\star_t$ within $8{\cdot}10^{-8}$ of $1$ on the two closed, uniformly sampled classes, while the epoch-$300$ network reads $0.050$--$0.055$ and the epoch-$50$ one $0.024$--$0.034$: an $18$ to $42\times$ deficit against a reference computed in the same harness at the same setting. At the CIFAR probe $\rhog^\star=0.1180$ and the measured Picard rate is $0.1409$--$0.1498$ --- one-sided high, exactly what $\rhog=\rhog^\star\sigma_t\lambda_1(J)$ with $\sigma_t\lambda_1>1$ predicts --- and at stride $20$ Picard measurably diverges where the Bayes map contracts at $0.7134$ (App.~\ref{app:exp-extra}). None of this is evidence \emph{for} the Bayes law: on an affine $\M$ the normal block of the noised \emph{empirical} log-density is data-independent~\citep[Eq.~(10)]{farghly2025manifold}, so a perfectly memorising score also reads $\Sc_t=\Sc^\star_t$, and the certificate is one-sided.

\paragraph{What the flattened dimension readouts measure.}
\label{sec:exp-lid}

A point-cloud sample is $n$ points in $\R^3$ flattened to $D{=}3n$, and the estimators of \S\ref{sec:method-readouts} measure $\dim\operatorname{supp}p_t$ \emph{in the ambient space they are run in}. Because $\epsb$ acts per point, $\partial(\epsb)_p/\partial x_q=0$ for $p\ne q$, so the probed law is a product whose support has dimension $nd$: the flattened target is $nd\in\{128,256\}$, not $d\in\{1,2\}$, and comparisons against $d$ are mis-specified. At the prescribed $K=4D=1536$ no rank readout on a trained model survives: the largest consecutive gap ratio is $1.02$--$1.58$ over all ten trained settings against a pure-noise null of $1.022$ at the same $(K,D)$, and every one abstains --- including the sphere, whose codimension-$128$ gap ratio of $5.06$ at $K=1.08D$ falls to $1.16$ at index $373$. That reading was a budget artefact: H5 is \emph{confirmed as stated} --- rank readouts do track their budget --- and \emph{refuted as a hope}, since raising $K$ to protocol retired our own positive along with the negative it was meant to rescue (App.~\ref{app:hyp}).

\paragraph{The oracle control.} Replacing $\epsb$ by the exact score under the identical harness, probe cloud, flattening and budget at $\sigma_t=0.0426$ gives $\hat d(\xo)=255.99999$, $255.99999$, $256.99$, $246.32$, $128.16$ on the sphere, two spheres, torus, swiss roll and trefoil against flattened targets $256,256,256,256,128$; the rank statistic recovers the true codimension on every closed class and the exponent split of \eqref{eq:nu} on four of five. The epoch-$300$ model reads $382$--$383$ on all five classes and the epoch-$50$ one $391$--$410$, i.e.\ the ambient $D$: the failure is a property of the trained score, not of the readout, and both oracle misses are the swiss roll's free boundary (App.~\ref{app:exp-extra}).

\paragraph{A negative that no change of target repairs.} The decisive test needs no ground truth: any hypothesis of the form ``the readout tracks $nd$'' predicts that the $d{=}1$ manifold reads exactly half of the $d{=}2$ manifolds. Over the $12$ pairwise settings that satisfy every validity guard --- the trefoil against each two-dimensional class separately, never against a mean --- the epoch-$300$ network's ratio lies in $[0.996,1.008]$ with mean $1.0003$, against an exact score reading $[0.496,0.536]$ on the very same settings; letting each class in turn play the $d{=}1$ role, the exact score puts the true one $16$ null standard deviations below the others while the trained model puts it $0.1$ \emph{above}. Being a ratio it is invariant to the ground truth and to any common multiplicative bias. The flattened amplitude readouts carry no class information about $d$ \emph{in this setting} (App.~\ref{app:exp-extra}).

\begin{figure}[t]
\centering
\includegraphics[width=0.8\linewidth]{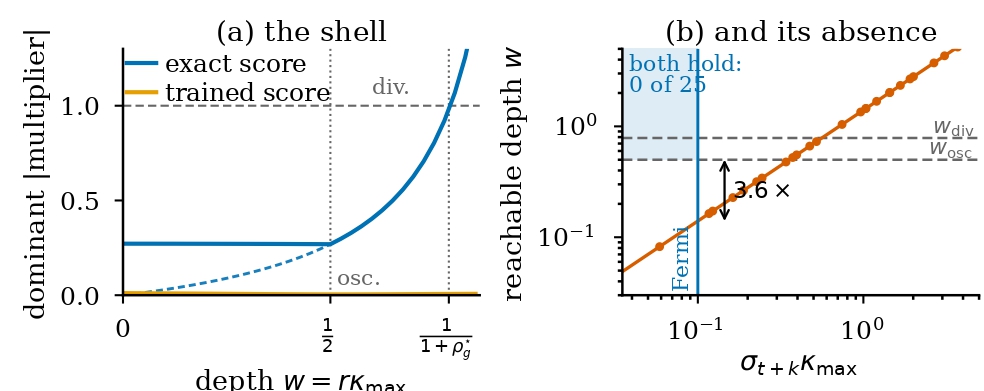}
\caption{The convergence domain, and why the trained shell is unmeasurable rather than absent. (a) Dominant Picard multiplier along an inward normal ray (sphere, $\sigma_{t+1}{=}0.0585$): on the exact score the negative branch (dashed) takes over at $w_{\rm osc}=0.4999$ and reaches $1$ at $w_{\rm div}=0.7891$, against the closed forms $\tfrac12$ and $0.7856$; the trained field is flat and crosses neither. (b) The Fermi window (vertical) caps the reachable depth at $1.4\,\sigma_{t+k}\kappa_{\max}$ (diagonal, one point per convergence setting), $3.6\times$ short of $w_{\rm osc}$: the hatched box where both hold is empty by arithmetic (\S\ref{sec:disc}).
}
\label{fig:shell}
\end{figure}

\paragraph{The convergence domain, measured.}
\label{sec:exp-shell}

On the exact score both shells sit where Thm.~\ref{thm:focal} puts them (Fig.~\ref{fig:shell}a). At $\sigma_{t+1}=0.0585$ the oscillation shell reads $w_{\mathrm{osc}}=0.499946$, $0.499997$ and $0.497288$ on the sphere, the two spheres and the torus against the schedule-free $\tfrac12$, and the divergence shell $w_{\mathrm{div}}=0.789124$ against $1/(1+\rhog^\star)=0.785615$, the residual $+0.003509$ being the predicted $\tfrac12\tilde\sigma_{t+1}^2\|\mathrm{I\!I}\|^2=0.003440$ --- a two-term correction the torus confirms to $0.8\%$. The three classes reading $\tfrac12$ are exactly the three whose largest principal curvature is constant over the cloud; the other two, whose curvature spreads threefold, return a \emph{divergence cliff} at $w_{\mathrm{div}}/w_{\mathrm{osc}}=1.05$ where a Bayes pair must be separated by $1.571$, so a single $\kappa_{\max}$ is not the supremum the law asks for and they fall outside its scope rather than contradict it (App.~\ref{app:exp-extra}, which also excludes the trefoil \emph{a priori}).

On every trained score the same sweep returns no shell, and that is a measurement rather than an absence. Of our $88$ trained shell readouts, $21$ sit at a setting where the exact score does resolve one; all $21$ are null, and the other $67$ carry no information because the exact score is null there too. The mechanism is measured three ways, on two independent draws of one recipe. On the convergence network the trained Jacobian is $11.4$--$13.7\times$ too weak in amplitude at the smallest setting; on the solver network, independently trained, the multiplier field is flat, at most $0.7113$ over all $1491$ admitted probes against a divergence threshold of $1$, and the enabling lemma is absent, reading $0.021$--$0.124$ against the exact score's $0.92$--$0.995$ when normalised by $\rhog^\star\kappa_{\max}$, and vanishing identically on a ReLU backbone. A score whose second derivative is a tenth of the curvature has a multiplier field that does not bend along the ray, so there is no shell to find --- a negative that is scale-free and has a closed-form Bayes value on both branches. Three qualifications are underpowered rather than negative and App.~\ref{app:hyp} states each: a suppressed flip, a divergence shell refuted where the oscillation shell is strictly \emph{unresolved}, and $\sigma_t$-exponents that are two-point slopes on one network. That no shell is found here is derived, not observed: the Fermi window and the model's own training support cap the reachable depth at $w\le0.14$ against shells at $w\ge\tfrac12$, jointly satisfiable in $0$ of our $25$ settings (Fig.~\ref{fig:shell}b; \S\ref{sec:disc}).

\paragraph{What the descent-based estimator measures.}
\label{sec:exp-theta}

The estimator used in practice initialises at $\N(0,I_D)$, runs fixed-step gradient descent on $E_t$ and regresses $\log\|\nabla E_t\|$ on $\log(E_t-\min_kE_t)$. By Remark~\ref{rem:descent} it must return $\tfrac12$ wherever the transverse spectrum is degenerate, whatever the model, and Fig.~\ref{fig:taut} confirms this on \emph{analytic} energies with no diffusion model at all: $\hat\theta=0.5000$ with $R^2=1.000000$ on the Morse--Bott energy under the divergent step the published hyper-parameters imply, and $0.077$ on a genuinely quartic energy whose true exponent is $\tfrac34$. It is a landscape-quadraticity detector with a mis-calibrated scale, not an estimator of $\theta$; the field estimator returns the true exponent on all four (App.~\ref{app:exp-extra}).

The image runs supply the counterexample. On pretrained CIFAR-10 and CelebA-HQ-256 --- converged by any external standard --- the descent protocol returns values indistinguishable from its own undertrained pole, and they are not merely inaccurate but \emph{inadmissible}: a non-constant $C^{1,1}$ function admits no {\L}ojasiewicz inequality with $\theta<\tfrac12$~\citep[Rem.~2.21]{rebjock2025fast}. The mechanism is the min-subtraction, which reproduces them with no model at all, so we do not report $\hat\theta$ as a training certificate. What \eqref{eq:hj} supplies instead is the field-based readout calibrated in \S\ref{sec:exp-calib}: sample $E_t$ and $\nabla E_t$ along a normal ray at $\xo$, taking the slope as a test of transverse quadratic structure and the intercept as $\sigma_t^2\hat\lambda\to1$. On exact scores it returns $\theta=0.50000$ and $\sigma_t^2\hat\lambda=0.9999$; on pretrained CelebA-HQ-256, $\theta=0.468$--$0.523$ with $R^2\ge0.9998$ on two of three images, against a descent estimate below $0.01$ on the same model, and along a random ray it does not ($0.69$--$0.95$). The intercept fails everywhere, so with $N=3$ we report the slope and not the amplitude, and at $t=4$ the fit is near-perfect at an inadmissible exponent --- a falsification with a clean fit rather than noise (App.~\ref{app:exp-extra}).

\paragraph{Validity windows.}
\label{sec:exp-window}
Every readout of \S\ref{sec:method} has a two-sided window in $\sigma_t$ whose upper edge depends on \emph{where} it is evaluated: at $\xo$ it is $\sigma_t\ll\reach(\M)$, while a \emph{sampling} estimator evaluates at $x_t$, whose distance to $\sqrt{\abar_t}\M$ concentrates at $\sigma_t\sqrt{D-d}$, a condition stricter by a factor $11$ at $D-d=128$ --- not new~\citep[Thm.~1]{niyogi2011topological}, but read here as an \emph{estimator-validity} criterion rather than a hypothesis of a recovery guarantee. Its effect is to \emph{remove} an excuse: all five classes are inside the window at the smallest probed noise once the product structure is accounted for, so the trefoil's failure to separate $d{=}1$ from $d{=}2$ leaves the protocol column and the oracle control above places it in the model column (App.~\ref{app:exp-extra}). On images the upper edge additionally requires $\reach(\M)$, which is estimable~\citep[Thm.~3.7, Alg.~1]{aamari2019reach,yaguchi2025geometry}.

%-----------------------------------------------------------------------

\section{Discussion and Conclusion}
\label{sec:disc}

\paragraph{Falsifiable predictions.} Four of our five are decided, and App.~\ref{app:hyp} gives each verdict with its measurement. \textbf{H1} is \emph{confirmed} end to end --- the gap to the exact score closes in $51$ of $55$ settings --- and \emph{refuted} as the fixed-$t$ monotonicity we stated it as, which holds in $5$ of $55$; the correct form is a front, moving inward by $2.09$--$2.29\times$. \textbf{H2} is \emph{confirmed} on every class satisfying Assumption~\ref{ass:tube} with a uniform base measure. \textbf{H4} is \emph{underpowered}, the patch-wise exponent reading $68$ of $192$ with $36\%$ of directions unresolved. \textbf{H5} is \emph{confirmed as stated and refuted as a hope}: raising $K$ to the prescribed $4D$ moved artefacts in both directions, retiring our own positive rank reading with the negative it was meant to rescue. \textbf{H3} was not run.

\paragraph{Limitations.} Assumptions~\ref{ass:tube}--\ref{ass:unimodal} are unverified on natural images, so our image statements are calibration measurements rather than dimension claims: exactly four of the readouts we report on images are hypothesis-free --- the ceiling $\sigma_t\lambda_{\max}(\mathrm{sym}\,J)\le1$, the skew index, $\sigma_t\tr J\le D$ and the sampled-point identity of Prop.~\ref{prop:vacuous} --- and those are precisely the four that fire. Every reading that needs a reach, a Fermi window or a tube edge is uncheckable there. $\Sc_t$ in particular is one-sided and, on images, unsound --- it fuses a positive branch that has an unconditional ceiling with a medial-axis branch that has none --- so we report the signed pair throughout and keep $\Sc_t$ only where the skew index licenses it.

\emph{The trained convergence shell is structurally unmeasurable by this probe, and that is a derived result rather than a null one.} Three requirements collide. The shells sit at $w_{\rm osc}=\tfrac12$ and $w_{\rm div}=1/(1+\rhog^\star)\simeq0.79$. Converting a shell radius into a curvature needs the Fermi window $\sigma_{t+k}\kappa_{\max}\le0.1$. And the probe must stay inside the model's training support, $r\lesssim1.4\,\sigma_{t+k}$. The last two bound the observable depth by $w=r\kappa_{\max}\le1.4\,\sigma_{t+k}\kappa_{\max}\le0.14$ --- short of the oscillation shell by $3.6\times$ and of the divergence shell by $5.6\times$, for \emph{any} manifold, \emph{any} schedule and \emph{any} training budget, because neither bound refers to the model. Measured: $0$ of $25$ trained settings satisfies both conditions, $\sigma_{t+k}\kappa_{\max}$ spans $0.059$ to $3.712$ against a $0.10$ window with one setting inside it, $18$ of $25$ satisfy even the looser $r\le2\sigma_{t+k}$ criterion of Table~\ref{tab:settings} alone, and in that one Fermi-admissible setting the predicted shells sit at $r/\sigma=8.5$ and $13.4$, six and ten times beyond the support. This is why every trained shell is reported against the exact score at matched noise and never against the closed form, and it is a statement about what the instrument can resolve --- not a finding that the trained score has no convergence-domain structure.

Three things we predicted and could not measure, each stated in App.~\ref{app:hyp} in the form we are willing to defend. Whether the convergence-domain law is observable on \emph{any} trained score is open, and our own measurement of why it is not observable here --- a Hessian--Lipschitz constant at $2$--$12\%$ of the curvature the law reads --- is a two-point $\sigma$-fit on one network with no error bar. The multiplicity certificate has a threshold verified in an analytic model and \emph{zero} observations on a trained score anywhere in this study. And the image failures stay unattributable for a measured reason: along the score ray the posterior covariance never exceeds $1.12\sigma_t^2$, so no shell's $\sigma$-exponent can be taken, and the point Prop.~\ref{prop:vacuous} forces sits $\sqrt D$ standard deviations inside the training noise shell --- the live confound, in place of the unknown reach. Our two training states are snapshots of one run and our two networks two draws of one recipe, so no trained class-to-class comparison here carries a population error bar. And we claim the focal-radius law for the Bayes limit and the oracle only: its multiplier statement holds at every $\sigma_t$ and every data law, but the finite-noise correction and any radius-to-curvature conversion hold only inside the Fermi window, and the law is observable only where $\kappa_{\max}$ is roughly constant over the probed set.

Two cheap measurements remain: the field-based $\theta$ on a held-out image set, and a reach estimator on natural images --- the prerequisite for any image LID claim including ours. Treating a one-step geometry probe as an instrument with a calibration curve changes what it can be asked. Every readout has a closed-form Bayes value, and verifying all of them against integrable manifolds converts each measured deviation into a statement about the trained model. Four consequences stand out. One inverse step is a strongly convex program whose modulus is a schedule constant, so its solution is unique and the solver's failures are step-size artefacts to be damped rather than evidence about the model. The contraction rate is that same schedule constant. The Fokker--Planck dimension functional, averaged over sampled points, measures a model's optimality gap and never its data manifold, which is why the noiseless evaluation point is forced. And the geometry the probe was built to find lives in the size of its own convergence domain --- where an exact score reproduces both closed-form boundaries, no trained score in this study produces either, and the reason it cannot is an inequality we compute rather than a limit of our budget.

%-----------------------------------------------------------------------
\phantomsection\addcontentsline{toc}{section}{References} % bookmark the bibliography
\bibliographystyle{iclr2026_conference}
\bibliography{refs}

\clearpage
\appendix

\section{The Calibration Suite}
\label{app:calib}

For a unit-radius manifold embedded at scale $\sqrt{\abar_t}$, with $z=r\sqrt{\abar_t}/\sigma_t^2$, the Gaussian convolution is available in closed form:
\begin{gather*}
S^1\subset\R^2:\ p_t\propto e^{-\frac{r^2+\abar_t}{2\sigma_t^2}}I_0(z);\qquad
S^2\subset\R^3:\ p_t\propto e^{-\frac{r^2+\abar_t}{2\sigma_t^2}}\frac{\sinh z}{z};\\
S^1\subset\R^3:\ p_t\propto e^{-\frac{\rho^2+x_3^2+\abar_t}{2\sigma_t^2}}I_0\!\Big(\frac{\rho\sqrt{\abar_t}}{\sigma_t^2}\Big),
\end{gather*}
with $\rho^2=x_1^2+x_2^2$ in the last case. Differentiating twice gives $\epss=-\sigma_t\nabla\log p_t$ and $\Jb=-\sigma_t\nabla^2\log p_t$ to machine precision; the Bessel ratio $R=I_1/I_0$ is evaluated in exponentially scaled form and differentiated with the Riccati identity $R'=1-R/z-R^2$.

\paragraph{Schedule sweep, and the two normalisations.} The contraction constant is $\rhog^\star=|B_t|/\sigma_{t+1}=1-e^{-h_t}$, with the score evaluated at $t{+}1$; normalising by $\sigma_t$ instead inflates it by exactly $e^{h_t}$, and the difference decides the theorem. On the standard DDPM linear schedule ($T{=}1000$, $\beta$ linear on $[10^{-4},0.02]$) the corrected constant reads $0.325699$ at $t{=}0$, $0.218175$ at $t{=}1$, $0.167878$ at $t{=}2$, $0.118045$ at $t{=}4$, $0.076633$ at $t{=}8$, $0.027992$ at $t{=}30$, $0.009941$ at $t{=}100$, $0.005022$ at $t{=}300$ and $0.010051$ at $t{=}998$, with maximum $0.3256994484136135$ attained at the first step --- against $0.482989$ for $\max_t|B_t|/\sigma_t$ under the old normalisation. The identity $\rhog^\star=1-e^{-h_t}$ holds to a worst residual of $9.410\times10^{-16}$ over six families $\times$ three strides, $9.402\times10^{-16}$ over eleven schedule configurations, $3.192\times10^{-16}$ over the whole point-cloud grid and $8.656\times10^{-16}$ over the DDPM $T{=}1000$ grid; the constant $0.32569944841361353$ is bit-identical across all three. The coarse schedules that broke the old reading do not break this one: $T{=}100$ gives $0.500687$ where the old form gave $1.003$, and $T{=}100$ with $\beta_{\rm end}=0.05$ gives $0.623219$ where the old form gave $1.654$.

\paragraph{Where the bound is attained.} Table~\ref{tab:sched} sweeps seven families at five strides, and three rows carry the scope of Thm.~\ref{thm:contract}. A cosine schedule has $\abar_T=0$ exactly, so $\rhog^\star\to1$ at its last step and the printed margin is set by the library's floor rather than by the schedule: we print the \texttt{diffusers} convention ($\beta$ clipped at $0.999$, giving $0.968377$), while clipping $\abar$ at $10^{-8}$ gives $0.935834$ for the same nominal family, and neither number is a property of the cosine schedule. Rectified flow reads $\rhog^\star=1.000000$ exactly at its top step, where $\abar=0$ and $h_t=\infty$: the bound is \emph{attained}, not violated, and the exact statement is $\rhog^\star<1$ for every schedule with $\abar_T>0$, with equality only at a zero-terminal-SNR top step~\citep{lipman2023flow}. EDM/Karras spans $0.024484$ at stride $1$ to $0.681146$ at stride $50$, so the certificate's margin varies by a factor $28$ across families~\citep{karras2022elucidating,nichol2021improved}. Parametrisation is immaterial by algebra rather than by measurement: over $N{=}400$ matched $\sigma$-grids $\max|h^{\rm VE}_t-h^{\rm VP}_t|=8.9\times10^{-16}$ and $\max|\rhog^{\star,\rm VE}-\rhog^{\star,\rm VP}|=2.9\times10^{-16}$ while the mean-value coefficient itself moves by $1.96\times10^{-2}$; SD3/Flux static shifts $m\in\{1,3,6\}$ and dynamic shifts $\mu\in\{0.5,1.0,1.9\}$ move $\rhog^\star$ by at most $3.8\times10^{-14}$, because a translation of the log-SNR axis cannot change any $h_t$.

\begin{table}[h]
\caption{Seven schedule families at five strides. $\rhog^\star=1-e^{-h_t}$ is the Bayes contraction rate at the family's worst step, $w_P=1/(1+\rhog^\star)$ the Picard divergence boundary of Thm.~\ref{thm:focal}, and the last column is the superseded $|B_t|/\sigma_t$ normalisation, shown to make the size of the correction visible. Only a zero-terminal-SNR top step attains $\rhog^\star=1$.}
\label{tab:sched}
\centering
\small
\begin{tabular}{@{}llccccc@{}}
\toprule
family & stride $k$ & $\rhog^\star$ & $1/\rhog^\star$ & $w_P$ & $|B_t|/\sigma_t$ (old) \\
\midrule
DDPM-linear $T{=}1000$ & 1 & 0.325699 & 3.070 & 0.7543 & 0.4830 \\
 & 10 & 0.786704 & 1.271 & 0.5597 & 3.6845 \\
 & 20 & 0.874046 & 1.144 & 0.5336 & 6.9180 \\
 & 50 & 0.943182 & 1.060 & 0.5146 & 16.3496 \\
 & 100 & 0.970781 & 1.030 & 0.5074 & 31.4357 \\
\addlinespace
cosine $T{=}1000$ (clip $.999$) & 1 & 0.968377 & 1.033 & 0.5080 & 0.9684 \\
 & 10 & 0.996838 & 1.003 & 0.5008 & 3.1796 \\
 & 20 & 0.998419 & 1.002 & 0.5004 & 5.7595 \\
 & 50 & 0.999369 & 1.001 & 0.5002 & 13.1621 \\
 & 100 & 0.999686 & 1.000 & 0.5001 & 25.2536 \\
\addlinespace
sigmoid $T{=}1000$ & 1 & 0.293621 & 3.406 & 0.7730 & 0.4156 \\
 & 10 & 0.701663 & 1.425 & 0.5877 & 2.3501 \\
 & 20 & 0.786489 & 1.271 & 0.5598 & 3.6779 \\
 & 50 & 0.868084 & 1.152 & 0.5353 & 6.5530 \\
 & 100 & 0.913396 & 1.095 & 0.5226 & 10.4442 \\
\addlinespace
scaled-linear SD $T{=}1000$ & 1 & 0.294020 & 3.401 & 0.7728 & 0.4163 \\
 & 10 & 0.703257 & 1.422 & 0.5871 & 2.3596 \\
 & 20 & 0.788630 & 1.268 & 0.5591 & 3.6976 \\
 & 50 & 0.870703 & 1.148 & 0.5346 & 6.5718 \\
 & 100 & 0.915195 & 1.093 & 0.5221 & 10.2094 \\
\addlinespace
EDM/Karras $N{=}1000$ (VP) & 1 & 0.024484 & 40.843 & 0.9761 & 0.0251 \\
 & 10 & 0.216531 & 4.618 & 0.8220 & 0.2764 \\
 & 20 & 0.381110 & 2.624 & 0.7241 & 0.6158 \\
 & 50 & 0.681146 & 1.468 & 0.5948 & 2.1362 \\
 & 100 & 0.880607 & 1.136 & 0.5317 & 7.3747 \\
\addlinespace
VE/Karras $N{=}1000$ & 1 & 0.024484 & 40.843 & 0.9761 & 0.0251 \\
 & 10 & 0.216531 & 4.618 & 0.8220 & 0.2764 \\
 & 20 & 0.381110 & 2.624 & 0.7241 & 0.6158 \\
 & 50 & 0.681146 & 1.468 & 0.5948 & 2.1362 \\
 & 100 & 0.880607 & 1.136 & 0.5317 & 7.3757 \\
\addlinespace
FlowMatchEuler $N{=}1000$ & 1 & \textbf{1.000000} & 1.000 & 0.5000 & 1.0010 \\
 & 10 & 1.000000 & 1.000 & 0.5000 & 10.0100 \\
 & 20 & 1.000000 & 1.000 & 0.5000 & 20.0200 \\
 & 50 & 1.000000 & 1.000 & 0.5000 & 50.0501 \\
 & 100 & 1.000000 & 1.000 & 0.5000 & 100.1001 \\
\bottomrule
\end{tabular}
\end{table}

\begin{figure}[h]
\centering
\includegraphics[width=\linewidth]{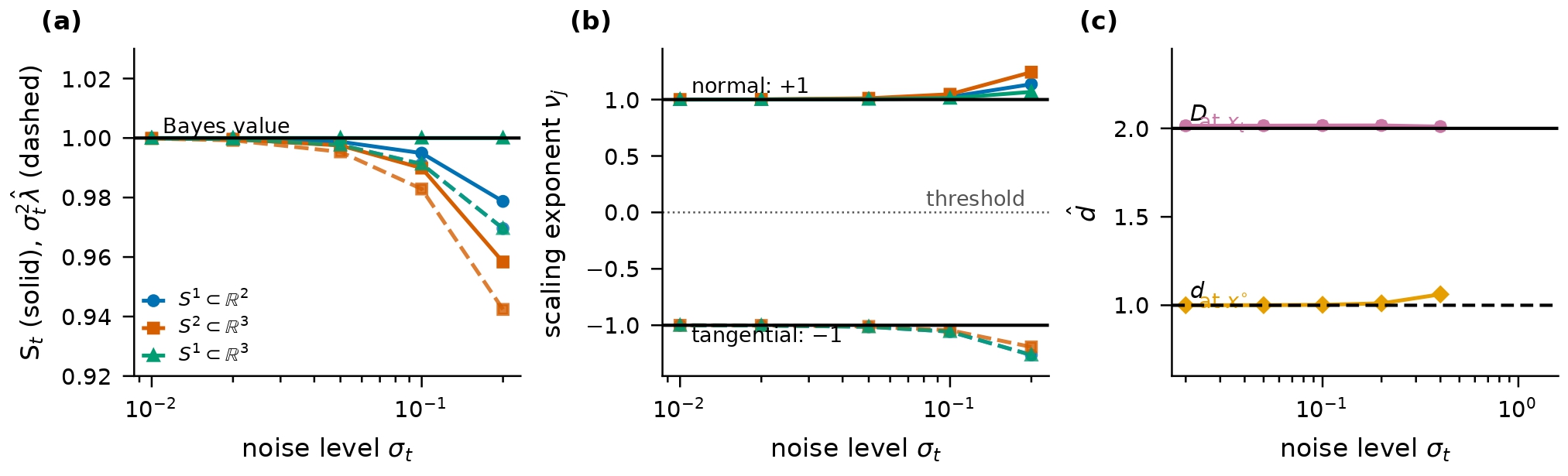}
\caption{Bayes-limit calibration on analytically integrable manifolds; no trained model is involved. (a) The two dimensionless certificates $\Sc_t$ (solid) and $\sigma_t^2\hat\lambda$ (dashed). The offset of $\Sc_t$ from $1$ is not error: it is the closed-form curvature deficit of \eqref{eq:stiff-cert}, matched to all printed digits in codimension one, with an $O(\sigma_t^4)$ residual. (b) The scaling exponents \eqref{eq:nu} converge to $+1$ (normal) and $-1$ (tangential); the value $0$ often assumed for the tangential family is the dotted threshold, not a limit. (c) The same functional evaluated at the sampled point returns the ambient $D$ at every noise level (Prop.~\ref{prop:vacuous}) and at the noiseless point returns the intrinsic $d$ (Prop.~\ref{prop:exact}); error bars are $95\%$ intervals over $2\times10^4$ draws.}
\label{fig:calib}
\end{figure}

\begin{figure}[h]
\centering
\includegraphics[width=\linewidth]{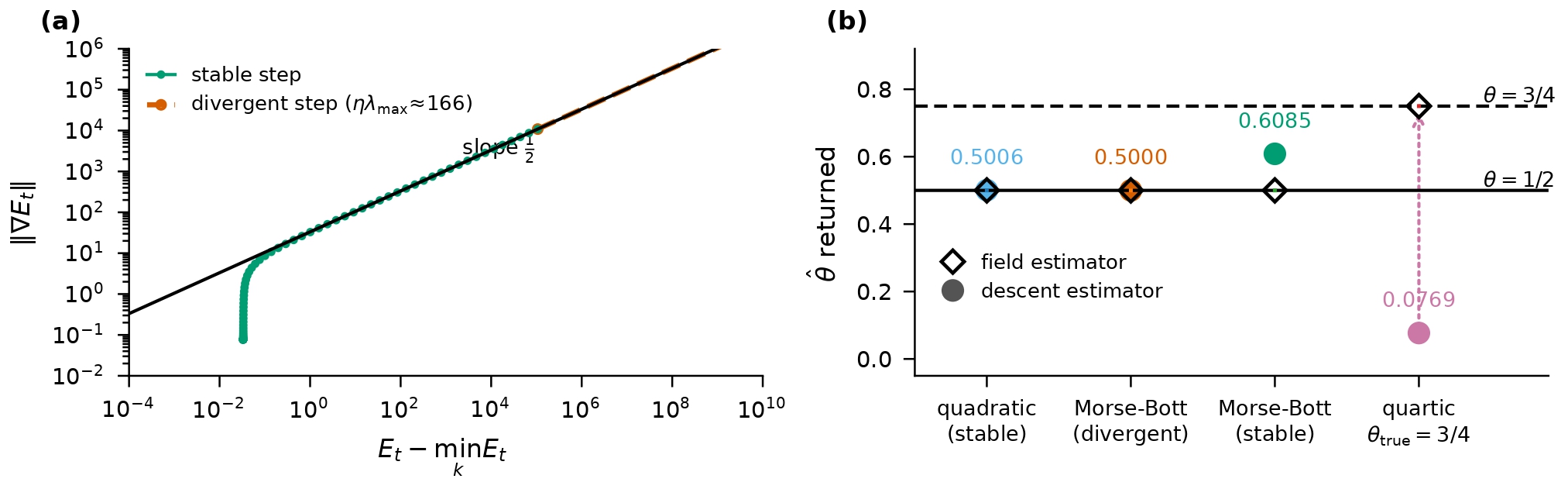}
\caption{The descent-based {\L}ojasiewicz estimator is uninformative by construction (Remark~\ref{rem:descent}). (a) For a quadratic energy the log--log locus is the line of slope $\tfrac12$ over ten decades whether the iteration converges or diverges; the two traces coincide with the theorem line. (b) What the estimator returns on analytic energies of known exponent: $\tfrac12$ on quadratic and Morse--Bott energies including under a divergent step, a different value under a stable step on the same energy, and $0.077$ on a quartic energy whose true exponent is $\tfrac34$ (dotted arrow to the truth). Open diamonds: the field estimator of \eqref{eq:hj} on the same energies, which returns the true exponent in all four cases.}
\label{fig:taut}
\end{figure}

\section{Full Proofs}
\label{app:proofs}

Throughout, $\epsb=\epss$, $A_t=\sqrt{\abar_{t+1}/\abar_t}$, $B_t=\sigma_{t+1}-A_t\sigma_t$, $\lambda_t=\log(\sqrt{\abar_t}/\sigma_t)$ and $h_t=\lambda_t-\lambda_{t+1}$, strictly positive for any strictly decreasing $\abar$. Everything rests on one exact schedule identity, which involves no expansion in the step size:
\[
A_t\sigma_t=\sigma_t\sqrt{\abar_{t+1}/\abar_t}=e^{-h_t}\sigma_{t+1},
\quad\text{hence}\quad
B_t=\sigma_{t+1}\big(1-e^{-h_t}\big),\quad \rhog^\star:=\frac{|B_t|}{\sigma_{t+1}}=1-e^{-h_t}.
\]
Since $h_t$ is a difference of log-SNRs, no reparametrisation of the noise axis and no shift of the log-SNR grid can change it; this is why the constant transfers unchanged between VP, VE and shifted flow-matching parametrisations (App.~\ref{app:calib}).

\begin{proof}[Proof of Theorem~\ref{thm:potential}]
By Tweedie--Miyasawa \eqref{eq:tweedie} at level $t{+}1$, $\epss(x,t{+}1)=-\sigma_{t+1}\nabla\log p_{t+1}(x)$, so
\[
F(x)=x-G(x)=x-A_ty-B_t\epss(x,t{+}1)=x-A_ty+B_t\sigma_{t+1}\nabla\log p_{t+1}(x)=\nabla\Psi_t(x)
\]
with $\Psi_t(x)=\tfrac12\|x-A_ty\|^2+B_t\sigma_{t+1}\log p_{t+1}(x)$; the residual is a gradient field for \emph{every} data law, because it is an affine function of a score. Differentiating once more and using $\Jb=-\sigma_{t+1}\nabla^2\log p_{t+1}$,
\[
\nabla^2\Psi_t(x)=I-B_t\Jb(x,t{+}1)
=I-\frac{B_t}{\sigma_{t+1}}\Big(I-\frac{C_{t+1}(x)}{\sigma_{t+1}^{2}}\Big)
=e^{-h_t}I+\rhog^\star\frac{C_{t+1}(x)}{\sigma_{t+1}^{2}},
\]
by \eqref{eq:Jcov} at $t{+}1$ and $1-\rhog^\star=e^{-h_t}$. Since $C_{t+1}\succeq0$ for every $p_0$, $\nabla^2\Psi_t\succeq e^{-h_t}I$, so $\Psi_t$ is strongly convex with modulus exactly $\mu=e^{-h_t}$ and $\mu+\rhog^\star=1$. Picard iteration $x\mapsto G(x)=x-F(x)=x-\nabla\Psi_t(x)$ is therefore gradient descent on $\Psi_t$ with step size exactly $1$ --- not approximately, and with no freedom to choose otherwise. Inside the noise tube of a uniform $\M$ the Fermi values of the posterior covariance are $c_\perp=0$ on the $D-d$ normal directions and $c_\parallel=\sigma_{t+1}^2/(1-r\kappa_i)$ on the $d$ tangential ones, so at $r=0$ the spectrum of $\nabla^2\Psi_t$ is $\{1^{(d)},(e^{-h_t})^{(D-d)}\}$: its condition number is $e^{h_t}$, independent of the data law, of $d$, of $D$ and of curvature, and $s_{\min}(I-B_t\Jb)=e^{-h_t}$ exactly. Convexity of $-\log p_{t+1}$ is not assumed anywhere; what is used is that $\log p_{t+1}$ is the log-partition function of the denoising exponential family, whose Hessian is a covariance.
\end{proof}

\begin{proof}[Proof of Proposition~\ref{prop:ceiling-sym}]
At Bayes $\Jb=-\sigma_t\nabla^2\log p_t$ is symmetric, so $\mathrm{sym}\,\Jb=\Jb$ and by \eqref{eq:Jcov}
$\sigma_t\lambda_{\max}(\Jb)=1-\lambda_{\min}(C_t)/\sigma_t^{2}\le1$, with equality iff $C_t$ is singular in the extremal direction. No manifold, reach, unimodality or tube hypothesis enters, and the statement is pointwise in $x$ and uniform in $t$. Two consequences are used in \S\ref{sec:exp}: $\sigma_t\lambda_{\max}(\mathrm{sym}\,J)>1$ falsifies Bayes optimality outright, and $\sigma_t\tr J\le D$ for the same reason, since $\sigma_t\tr\Jb=D-\tr(C_t)/\sigma_t^2$.
\end{proof}

\begin{proof}[Proof of Theorem~\ref{thm:unique}]
By Theorem~\ref{thm:potential}, $\nabla^2\Psi_t\succeq e^{-h_t}I\succ0$ for every $x$, so $\Psi_t$ is strictly convex and $F=\nabla\Psi_t$ has exactly one zero. Replacing $\epss$ by a trained $\epsb$ destroys the gradient structure but not the argument: $F_\phi(x)=x-A_ty-B_t\epsb(x,t{+}1)$ satisfies
$\langle F_\phi(x)-F_\phi(z),x-z\rangle\ge\big(1-B_t\lambda_{\max}(\mathrm{sym}\,J_\phi)\big)\|x-z\|^2$
along the segment, so $F_\phi$ is strictly monotone --- hence injective --- unless
$\sigma_{t+1}\lambda_{\max}(\mathrm{sym}\,J_\phi(\xi,t{+}1))>1/(1-e^{-h_t})$ for some $\xi$. The condition is therefore necessary for a second solution and, being a bound on the symmetric part, not sufficient. At the threshold the fold is the only bifurcation available: a multiplier $+1$ of $DG=B_tJ_\phi$ is exactly the displayed inequality, a multiplier $-1$ produces period doubling but no second root, and a complex pair is excluded whenever $DG$ is symmetric, which it is at Bayes and approximately so wherever the skew index is small. In the two-atom model the new roots sit at $x=\rho^\star c/(\rho^\star c-1)$, where the posterior covariance vanishes and $|G'|\to\rho^\star c$ --- the threshold quantity itself --- so they are repelling exactly when they exist. Multiplicity is thus a Newton observable and is invisible to Picard, which is why we report it through the signed pair and not through a failed iteration.
\end{proof}

\begin{proof}[Proof of Theorem~\ref{thm:focal}]
Combine Theorem~\ref{thm:potential} with the Fermi values of $C_{t+1}$. The Picard multipliers are the eigenvalues of $I-\nabla^2\Psi_t=\rhog^\star\big(I-C_{t+1}/\sigma_{t+1}^2\big)$, real because $\nabla^2\Psi_t$ is symmetric. Displacing inward along a principal normal by $r$ and writing $w_i=r\kappa_i$, the exact tube values $c_\perp=0$ and $c_{\parallel,i}=\sigma_{t+1}^2/(1-w_i)$ give the multiplier set
$\{+\rhog^\star\ (\times\,D{-}d)\}\cup\{-\rhog^\star w_i/(1-w_i)\ (\times\,d)\}$
for every $\sigma_t$ and every data law. (1) The dominant multiplier changes sign when $w/(1-w)=1$, i.e.\ $w_{\rm osc}=\tfrac12$; the same condition is $\lambda_{\max}(C_{t+1})=2\sigma_{t+1}^2$, which is exactly the failure of Assumption~\ref{ass:unimodal}, so the oscillation boundary and the unimodality edge are one object and neither depends on the schedule. (2) The dominant multiplier reaches magnitude $1$ when $\rhog^\star w/(1-w)=1$, i.e.\ $w_{\rm div}=1/(1+\rhog^\star)$; equivalently $\lambda_{\max}(\nabla^2\Psi_t)=e^{-h_t}+\rhog^\star/(1-w)=2$, so the Picard divergence boundary and the damping threshold $\eta<2/\Lambda_t$ are the same equation. Beyond the statement, differentiating $\lambda_{\max}(\nabla^2\Psi_t)$ in $r$ gives the ray-Lipschitz constant $L_{\mathrm{H}}(w)=\rhog^\star\kappa_{\max}/(1-w)^2$, while $\|(\nabla^2\Psi_t)^{-1}\|\le e^{h_t}$ everywhere by Theorem~\ref{thm:potential}, so the Newton--Kantorovich product is $aL_{\mathrm{H}}=c\,\kappa_{\max}/(1-w)^2$ with $c=e^{h_t}-1$: free of the data law, of the density and of the dimension, and divergent as $w\to1$, so no certified ball reaches the focal point. (3) Carrying the next order of the Fermi determinant, the correction to $w_{\rm div}$ is quadratic in $\tilde\sigma=\sigma/\sqrt{\abar}$ with coefficient the trace of the squared second fundamental form, and the correction to $w_{\rm osc}$ is quartic; the shape constant is fixed by measurement rather than by this sketch, and \S\ref{sec:exp-calib} reports it at $1.00007$--$1.00178$ on the closed-form targets and $1.008$--$1.020$ on real clouds. Only $\|\mathrm{II}\|^2$ reproduces both the codimension and the dimension dependence: on a curve $\|\mathrm{II}\|^2=\kappa^2$ and the two candidate normalisers coincide, which is why the $27$-setting scan over $S^1\subset\R^2$, $S^2\subset\R^3$ and $S^1\subset\R^3$ is the only place in this study that can separate them.
\end{proof}

\begin{proof}[Proof of Theorem~\ref{thm:contract}]
$B_t=\sigma_{t+1}(1-e^{-h_t})$ is exact, so $\rhog^\star=1-e^{-h_t}<1$ for every schedule with $h_t>0$, i.e.\ for every strictly decreasing $\abar$ with $\abar_T>0$; the earlier closed form was an expansion of a quantity that needs none. The Lipschitz factorisation is definitional, $\rhog=|B_t|\sup_\xi s_1(J(\xi,t{+}1))=\rhog^\star\Sc_{t+1}$, and $\Sc_{t+1}\le1$ under Assumption~\ref{ass:unimodal} by \eqref{eq:Jcov}. The damping statement follows from Theorem~\ref{thm:potential}: gradient descent on an $e^{-h_t}$-strongly-convex, $\Lambda_t$-smooth $\Psi_t$ converges for every $\eta<2/\Lambda_t$, optimally at $\eta^\star=2/(e^{-h_t}+\Lambda_t)$ with rate $(\Lambda_t/\mu-1)/(\Lambda_t/\mu+1)$ in the condition number $\Lambda_t/\mu$; inside the tube at $r=0$ that number is $e^{h_t}$ and the rate is $\tanh(h_t/2)$. Finally $\mu+\Lambda(\tfrac12)=2$ identically, so unit-step Picard is \emph{optimally} damped precisely at the oscillation shell and nowhere else.
\end{proof}

\paragraph{Three levels, and why only one of them carries geometry.} The three statements above separate objects that the literature treats as one. The \emph{solution set} is trivial at the Bayes limit --- one root, by strict convexity, for every data law and every schedule --- so a second solution is a hypothesis-free certificate of model error and carries no information about $\M$. The \emph{solver orbit} can be pathological at the Bayes limit and certifies nothing about the model: on an exact two-atom Bayes score with a unique minimiser and $\nabla^2\Psi_t\succeq0.4I$ everywhere, $\lambda_{\max}(\nabla^2\Psi_t)=2.011,3.022,4.033,7.067$ at $\rhog^\star=0.1,0.2,0.3,0.6$, and undamped Picard reaches tolerance from $1$ of $2001$ starts on $[-6,6]$ --- the one being the fixed point itself, so $0$ of $2000$ non-trivial starts --- while it is captured by an attracting $2$-cycle at $\pm0.374819$; at $\eta=1/\Lambda_t$ the same sweep converges from $100.00\%$ of starts. What does carry geometry is the \emph{convergence domain}: its two boundaries are the closed forms of Theorem~\ref{thm:focal}, and they are curvatures. A binary convergence readout is the wrong instrument for the first two levels and the right one only for the third, and even there its rate vanishes at $\Lambda_t=2$, so the sharp object is the theorem and not the iteration count.

\paragraph{Numerical verification of Thm.~\ref{thm:potential}.} Two atoms at $\pm1$, $\sigma=0.3$, $200\,001$ grid points: the Hessian identity holds to $8.9\times10^{-16}$, $1.8\times10^{-15}$, $1.8\times10^{-15}$ and $3.6\times10^{-15}$ at $\rhog^\star=0.2,0.4,0.6,0.8$; $\mu=e^{-h_t}$ to $2.8\times10^{-17}$; $\mu+\rhog^\star=1.000000$ on $4$ of $4$ rows; and $F=\nabla\Psi_t$ to the central-difference floor, $\le4.5\times10^{-11}$ at $dx=10^{-5}$, with clean $O(dx^2)$ behaviour from $10^{-2}$ down to $10^{-5}$. On the Bessel circle at $\sigma=0.01$ the two eigenvalues are $0.99998$ and $0.67432$ against $e^{-h_t}=0.67430$. The first Kantorovich constant is attained: $a/e^{h_t}=0.9947$ over ten exact-score settings, range $0.98751$--$0.99932$.

\paragraph{Numerical verification of Thm.~\ref{thm:unique}.} The threshold is verified in an analytic model only, never on a trained score. At Bayes, $\sigma_{t+1}\lambda_{\max}(\Jb)=1.000000$ and there is exactly one root at $\rhog^\star=0.05,0.20,0.50,0.90,0.99$ with $\min\nabla^2\Psi_t=0.95,0.80,0.50,0.10,0.010$. With an amplitude-scaled score $\epsb=\gamma\,\epss$ at $\rhog^\star=0.6$ the root count is $1$ at $\rhog^\star\gamma=0.600,0.840,0.960,0.996$ and $3$ at $1.050,1.200,1.800$, bracketing the predicted $1$; the new roots sit at $\pm21.00,\pm6.00,\pm2.25$, i.e.\ at $x=\rhog^\star\gamma/(\rhog^\star\gamma-1)$ exactly, with multipliers $\rhog^\star\gamma=1.0500,1.2000,1.8000$ --- so they are repelling iff they exist. Two limits of that experiment must be printed with it: the root search spans $|x|\le60$, so onset is undetectable for $\gamma\le1.695$ and the bracket $\gamma\in(1.66,1.75]$ cannot be tightened in this configuration; and the planned pitchfork-on-a-network section was not built, so the certificate has zero observations on any trained score.

\paragraph{Numerical verification of Thm.~\ref{thm:focal}.} On the Bessel-exact circle, $w_{\rm osc}=0.4999999599,0.4999999975,0.4999999998$ at $\sigma=0.02,0.01,0.005$, identical at $\rhog^\star=0.278$ and $0.6527$, with residual ratios $16.03$ and $16.26$, i.e.\ $O(\sigma^4)$, while the $w_{\rm div}$ residual falls with ratios $4.005$ and $4.001$, i.e.\ $O(\sigma^2)$, and the coefficient of part (iii) is measured at $1.00178$ falling to $1.00007$. The $\|\mathrm{II}\|^2$ normalisation, and not $\kappa_{\max}^2$, is what the data selects: over a $27$-setting scan on $S^1\subset\R^2$, $S^2\subset\R^3$ and $S^1\subset\R^3$ at three noise levels and three schedules, the two-term form reproduces every measured $w_{\rm div}$ at the scan's own resolution, $27$ of $27$, where the one-term form misses $8$ of $27$; $S^1\subset\R^3$ is identical to $S^1\subset\R^2$ in $9$ of $9$ settings, so codimension has no effect, and $S^2$ sits one grid step high in exactly the two settings the doubled coefficient predicts.

\section{Deferred Method Material}
\label{app:method-extra}

This section carries the proofs and the attributions of \S\ref{sec:method-readouts}; the statements themselves are in the main text.

\paragraph{The Fokker--Planck readout at the noiseless point.}

\begin{proposition}[The noiseless-point readout]\label{prop:exact}
At Bayes $\hat d(x)=\E[\|\eps\|^2\mid x_t{=}x]$ pointwise, and
\[
\hat d(\xo)=d+\frac{\sigma_t^{2}}{2\abar_t}\Big(\|\mathrm{II}\|^{2}-\tfrac12\|H\|^{2}\Big)
+\frac{\sigma_t^{2}}{\abar_t}\frac{\Delta_{\M}\rho_0}{\rho_0}+O(\sigma_t^{4}).
\]
The residual is $\epss(\xo)=-\tfrac{\sigma_t}{\sqrt{\abar_t}}\big(\nabla_\M\log\rho_0+\tfrac12H\big)+O(\sigma_t^2)$, which is $-\tfrac{\sigma_t}{2\sqrt{\abar_t}}H+O(\sigma_t^2)$ for a uniform $\rho_0$, but the trace term contributes at the \emph{same} order, so the total correction is \emph{signed}: $+\sigma_t^2\kappa^2/(4\abar_t)$ for a curve of curvature $\kappa$, exactly $0$ for the round $S^2$, and $-\tfrac{d(d-2)}{4}\sigma_t^2/\abar_t$ for $S^d\subset\R^{d+1}$, negative for $d\ge3$.
\end{proposition}
\begin{proof}
From the Fermi expansion of \S\ref{sec:method-normal}, $\log p_t=-r^2/2\sigma_t^2-\tfrac12\log\det(I-rA_\nu)+\log\rho_0+\mathrm{const}$. Differentiating and applying \eqref{eq:tweedie} gives a normal component $r/\sigma_t-\tfrac{\sigma_t}{2}\tr A_\nu+O(\sigma_tr)$, and $\tr A_\nu$ contracted with $\nu$ is $H$; at $r=0$ only that term survives. The trace term contributes $D-d$ at leading order by \eqref{eq:normal}.
\end{proof}

\paragraph{Proof of Prop.~\ref{prop:vacuous}.} Integration by parts --- the implicit score-matching identity~\citep{hyvarinen2005estimation} --- gives $\E_{p_t}[-\sigma_t\tr J]=-\E\langle\epsb,\epss\rangle=-\E\langle\epsb,\eps\rangle$, so $\E[\hat d]=D+\E\langle\epsb,\epsb-\epss\rangle=D-\delta_t/2$ by \eqref{eq:delta}. At Bayes $\hat d=D+\sigma_t^2\Delta p_t/p_t$ pointwise and $\int_{\R^D}\Delta p_t=0$ by the divergence theorem.

\paragraph{Two consequences of the vacuity.} Averaged over sampled points the readout is biased upward from $D$ by exactly $\gamma(\gamma{-}1)\Pi^\star_t+\E\|e\|^2$ (Prop.~\ref{prop:delta}), and $\E_{p_t}[-\sigma_t\tr J]=-\E\langle\epsb,\eps\rangle$ makes Hutchinson redundant \emph{in expectation} there, one dot product with the injected noise replacing $32$ probes. None of this is a criticism of the published protocol but the reason for its design: \citet[Eq.~13]{kamkari2024geometric} evaluate at the scaled clean point $\psi(t)x=\xo$, and Prop.~\ref{prop:vacuous} says that choice is forced. Pointwise the statement is weaker, since $\hat d(x_t)=D+\sigma_t^2\Delta p_t/p_t$ is not constant; the obstruction is to the average, and \S\ref{sec:exp-window} supplies a second one at the same point.

\paragraph{The DSM readout.} $L_t$ is the denoising-score-matching objective in its noise parametrisation~\citep{vincent2011connection}.

\begin{proposition}\label{prop:delta}
The two readouts $L_t/\abar_t$ and $(D-\Pi_t)/\abar_t$ differ exactly by $\delta_t/\abar_t$ and cannot corroborate each other where $\delta_t\ne0$. Writing $\epsb=\gamma\epss+e$ with $\E\langle\epss,e\rangle=0$, the decomposition is exact:
$\delta_t=-2\gamma(\gamma-1)\Pi^\star_t-2\E\|e\|^2 .$
\end{proposition}

\paragraph{The rank ceiling.}

\begin{proposition}\label{prop:ceiling}
Local PCA on $K$ Tweedie-projected samples satisfies $\hat d\le K-1$; centring is intrinsic to that estimator. A score matrix built from $K$ score evaluations caps the estimated \emph{codimension} at $\widehat{D-d}\le\min(D,K)$, and at $K-1$ if the evaluations are centred first, hence $\hat d\ge D-K+1$ in the centred case.
\end{proposition}
\begin{proof}
The estimated subspace is the span of the $K$ evaluations --- the tangent space for local PCA, the normal space for the score matrix --- so its dimension is at most $K$; centring makes the $K$ columns sum to zero and removes one further degree of freedom.
\end{proof}

\noindent Pinned at the ceiling, local PCA reports a lower bound on $d$ and the score-rank estimator an \emph{upper} bound; neither binds at the prescribed $K=4D$.

\paragraph{Morse--Bott, not Morse.}

\begin{remark}\label{rem:mb}
$E_t$ does not vanish on $\sqrt{\abar_t}\M$: by Prop.~\ref{prop:exact}, $E_t(\xo)=\sigma_t^2\|H\|^2/(8\abar_t)+O(\sigma_t^3)$, and the zero set of $\epss$ is the curvature-shifted manifold $\M^\star_t=\{\sqrt{\abar_t}x_0+\tfrac{\sigma_t^2}{2}H(x_0)+O(\sigma_t^3)\}$. Along $\M^\star_t$ the Hessian has a $d$-dimensional kernel, so $E_t$ is Morse--Bott and never Morse, and any derivation assuming a positive-definite Hessian is vacuous here. Transverse to $\M^\star_t$ it is $(I-P)/\sigma_t^2+O(1)$, whose $D-d$ eigenvalues are equal to leading order --- exactly equal for the limiting energy $\tfrac12\dist^2$, diagonal along a normal ray in one offset-independent frame with all normal eigenvalues exactly $1$~\citep{ambrosio1998distance}, splitting only at relative order $\sigma_t^2$. It is this \emph{transverse} nondegeneracy --- automatic in the noise tube and unaffected by that splitting --- that delivers the exponent of \eqref{eq:hj}. That Morse--Bott forces the optimal exponent is known abstractly~\citep{feehan2017optimal,rebjock2025fast}; \eqref{eq:hj} adds the constant, and \eqref{eq:stiff} turns the intercept of a log--log fit into a dimensionless stiffness with Bayes value $1$.
\end{remark}

\paragraph{The Fokker--Planck functional} In DDPM parametrisation the functional of \citet{kamkari2024geometric} is
\begin{equation}
\hat d(x)
= \underbrace{D}_{\text{ambient}}
- \underbrace{\sigma_t\tr J(x,t)}_{\substack{\text{removes the codim.\ iff}\\ \epsb\ \text{near-projects}}}
+ \underbrace{\|\epsb(x,t)\|^2}_{\substack{\text{curvature term, }=\\ (\sigma_t\|H\|/2)^2\text{ at Bayes}}},
\label{eq:fp}
\end{equation}
with $\tr J$ by Hutchinson probes~\citep{hutchinson1989stochastic}. At Bayes $\hat d=D+\sigma_t^2\Delta p_t/p_t$, the form that makes Props.~\ref{prop:exact} and \ref{prop:vacuous} immediate.

\paragraph{The optimality gap} With shared noise draws and $\E\|\eps\|^2=D$, expanding \eqref{eq:LPidelta} gives exactly
\begin{equation}
\delta_t = D-\Pi_t-L_t = 2\E\langle\epsb,\eps-\epsb\rangle = 2\E\langle\epsb,\epss-\epsb\rangle,
\label{eq:delta}
\end{equation}
the last step by the tower property, so $\delta_t$ is the orthogonality defect of $\epsb$ against the $L^2$ projection $\epss$. The dimensionless stiffness read off the Hamilton--Jacobi intercept is
\begin{equation}
\sigma_t^2\hat\lambda:=\sigma_t^2\cdot\tfrac12e^{2\hat c}\ \longrightarrow\ 1 .
\label{eq:stiff}
\end{equation}

\paragraph{Rank readouts} Pinned at the ceiling of Prop.~\ref{prop:ceiling}, local PCA reports a lower bound on $d$ and the score-rank estimator an \emph{upper} bound; neither binds at $K=4D$. The largest-\emph{difference} rule is maximised at the top edge whenever $s_1$ detaches from the bulk; the \emph{ratio} rule is maximised at the bottom edge whenever $s_{\min}\to0$, automatic at $K\le D$ by Prop.~\ref{prop:ceiling} but not only there. The pure-noise null bounds \emph{false} gaps and leaves power the separate question H5 settles.

\paragraph{The score energy as an action.} At Bayes $E_t(x)=r^2/(2\sigma_t^2)+O(\sigma_t)$ is, to leading order, the squared distance to the manifold rescaled by the noise --- the classical action of the diffusion at time $\sigma_t^2$ (see \citealp{calin2024}, \S6.5.5, for this identification in the Riemannian setting), and inside the reach the nearest-point projection is unique with $\|\nabla r\|=1$, the eiconal property~\citep{federer1959curvature}.

\paragraph{The transverse law and the scaling exponents (\S\ref{sec:method-readouts}).} The last identity is transverse. Inside the reach $r$ is $C^1$ with $\|\nabla r\|=1$~\citep{federer1959curvature} and $E_t=r^2/(2\sigma_t^2)+O(\sigma_t)$, so with only transverse nondegeneracy
\begin{equation}
\|\nabla E_t\|=\frac{r}{\sigma_t^2}=\frac{\sqrt{2E_t}}{\sigma_t}
\ \Longrightarrow\ \theta=\tfrac12,\ \ c_{\mathrm L}=\frac{\sqrt2}{\sigma_t}
\label{eq:hj}
\end{equation}
on the whole critical manifold, and with its constant. $E_t$ is Morse--Bott there and never Morse, so any derivation assuming a positive-definite Hessian is vacuous (Rem.~\ref{rem:mb}), and the intercept of a log--log fit gives a dimensionless stiffness $\sigma_t^2\hat\lambda\to1$.

\begin{remark}[Why a descent cannot see this]\label{rem:descent}
For $E=\tfrac12\lambda r^2$ under any linear iteration $r_{k+1}=(1-\eta\lambda)r_k$, $\log\|\nabla E_k\|=\tfrac12\log E_k+\tfrac12\log(2\lambda)$ \emph{identically}, converging or diverging. A descent estimator of $\theta$ therefore returns $\tfrac12$ with $R^2=1$ whenever a single scale governs the active spectrum --- automatic here, the transverse eigenvalues coinciding --- and otherwise reports the step size and the iteration budget. Equation~\eqref{eq:hj} must be read from the \emph{field} along a normal ray.
\end{remark}

Finally, by \eqref{eq:normal} the normal singular values of $\Jb(\xo)$ scale as $\sigma_t^{-1}$ and the tangential ones as $\sigma_t^{+1}$, not $\sigma_t^{0}$, so matching singular values across two levels $\sigma_A<\sigma_B$ gives exponents
\begin{equation}
\nu_j:=\frac{\log\big(s_j(\sigma_A)/s_j(\sigma_B)\big)}{\log(\sigma_B/\sigma_A)}\ \longrightarrow\ \pm1\quad(\text{normal}/\text{tangential}),
\label{eq:nu}
\end{equation}
separated by $2$, with natural threshold $0$ and a classifier invariant to any state- and noise-independent amplitude error. Two conditions bind it: a random $m$-dimensional probe subspace meets the tangent space only if $m>D-nd$, so $\nu$ must be read on a full basis, a random $m\ll D$ probe pinning every exponent at $+1$ by construction; and at a free boundary the outward tangential direction is truncated and reads as normal, where the exact score returns $174$ normal directions against a true codimension of $128$. We report $\nu$ as family medians, never as counts.

\paragraph{Why the stiffness deficit is measurable only in low codimension.} $\Sc^\star_t$ of \eqref{eq:stiff-cert} vanishes iff some normal direction has vanishing shape operator, which is \emph{automatic} in high codimension: $\min_\nu\|A_\nu\|_F^2=\lambda_{\min}(Q)$ for $Q_{ab}=\langle\mathrm{II}^a,\mathrm{II}^b\rangle$, and $\mathrm{II}^a\in\mathrm{Sym}(d)$ forces $\rank Q\le d(d+1)/2$, so $\Sc^\star_t=1$ \emph{exactly} whenever $d(d+1)/2<D-d$ --- at $D{=}3072$, for every $d\le76$. \citet{wenliang2023score} already read a trained score Jacobian's spectrum against this scale at a fixed tolerance; \eqref{eq:stiff-cert} is the closed form of what that tolerance stands in for.

%-----------------------------------------------------------------------
\section{Deferred Derivations}
\label{app:deriv}

\paragraph{Exact DSM readout on a linear manifold.} With $\M=\mathrm{range}(U)$, $U\in\R^{D\times d}$ orthonormal, and $x_0=Uz$, $z\sim\N(0,I_d)$: $\Sigma_t=\abar_tUU^\top+\sigma_t^2I$ and $\epss(x)=\sigma_t\Sigma_t^{-1}x$. Each normal direction contributes $0$ to the loss ($\epss_\perp=\eps_\perp$ exactly) and each tangential direction contributes $\E((1-\sigma_t^2)\eps_\parallel-\sigma_t\sqrt{\abar_t}z)^2=\abar_t^2+\sigma_t^2\abar_t=\abar_t$. Hence $L^\star_t=\abar_t d$ exactly at every $t$, so $L_t/\abar_t=d$ on this model. The same computation gives $\hat d(\xo)=\abar_t(d+\sigma_t^2\|z\|^2)$, so at the typical radius the schedule-corrected form of \eqref{eq:fp} is $\hat d(\xo)/[\abar_t(2-\abar_t)]$, which is the schedule-corrected normalisation used throughout \S\ref{sec:exp}.

\paragraph{Amplitude decomposition.} With $\epsb=\gamma\epss+e$, the amplitude multiplier of \S\ref{sec:method-potential}, and $\E\langle\epss,e\rangle=0$: $\Pi_t=\gamma^2\Pi^\star_t+\E\|e\|^2$ and $\E\langle\epsb,\eps\rangle=\gamma\Pi^\star_t$, so $\delta_t=-2\gamma(\gamma-1)\Pi^\star_t-2\E\|e\|^2$ exactly. Both terms are non-positive for $\gamma>1$, which is why $\delta_t$ screens but does not test.

\paragraph{Descent tautology.} For $E=\tfrac12\lambda r^2$ and $r_{k+1}=(1-\eta\lambda)r_k$: $E_k=\tfrac12\lambda r_k^2$ and $\|\nabla E_k\|=\lambda r_k$, so $\log\|\nabla E_k\|=\tfrac12\log E_k+\tfrac12\log(2\lambda)$ for every $k$, independently of $\eta$ and of whether $|1-\eta\lambda|<1$. The fit is exact, so $R^2=1$ by construction. In general, with $q_k:=\|\nabla E_k\|^2/(2E_k)\in[\lambda_{\min},\lambda_{\max}]$, $\log\|\nabla E_k\|=\tfrac12\log E_k+\tfrac12\log(2q_k)$ identically for every spectrum and step size, so a log--log regression returns $\tfrac12+\Delta\log q/(2\Delta\log E)$ and deviates from $\tfrac12$ by at most $\log\kappa$ over twice the fitted range of $\log E$. The published hyper-parameters imply $\eta\lambda_{\max}\approx166$, i.e.\ $83\times$ the stability threshold $2/\lambda_{\max}$, and across $\eta\lambda_{\max}$ from $0.5$ to $166$ the readout stays in $[0.4974,0.5009]$ with $R^2\ge0.9998$. The same cancellation holds verbatim for $E=\tfrac12\sum_i\lambda_ir_i^2$ whenever the modes with $r_i\ne0$ share one $\lambda$ --- the transversally isotropic case of Remark~\ref{rem:mb}, where the intercept returns the single transverse eigenvalue $1/\sigma_t^2$ at any step size. With a spread transverse spectrum the log--log locus is governed asymptotically by the slowest-decaying mode, so the intercept returns $\lambda_j$ with $j=\arg\max_i|1-\eta\lambda_i|$: the \emph{largest} eigenvalue under a divergent step and the \emph{smallest} under a stable one. Equation~\eqref{eq:stiff} is read from the field along a normal ray, where no step size enters.

\paragraph{Assumption~\ref{ass:unimodal} as geometry.} In Fermi coordinates at second order about a uniform $\M$, the posterior at offset $r$ along the normal $\nu$ has $\lambda_{\max}(C_t)=\sigma_t^2/(1-r\kappa_{\max})+O(\sigma_t^4)$, so $\lambda_{\max}(C_t)\le2\sigma_t^2$ iff $r\le1/(2\kappa_{\max})$: half the focal radius, which coincides with half the reach only when the reach is curvature-attained. Quadrature on the uniform circle gives $\lambda_{\max}(C_t)/\sigma_t^2=1.3330,1.6661,1.9992,2.4987,3.9968$ at $r=0.25,0.40,0.50,0.60,0.75$ against the predicted $4/3,5/3,2,5/2,4$. Near a bottleneck the exact value saturates at $(\ell/2)^2$ rather than diverging, so the assumption fails only inside a slab of half-width $O(\sigma_t^2)$ about the medial sheet. Separately, log-concavity of the law of $\sqrt{\abar_t}x_0$ passes to the posterior with $\nabla^2(-\log)\succeq\sigma_t^{-2}I$, so Brascamp--Lieb gives $C_t\preceq\sigma_t^2I$ and margin $2$.

\paragraph{A lower edge for curvature readouts.} The Bayes residual at the noiseless point is $(r/\sigma_t)\nu-(\sigma_t/2)H$, so the two terms cross at $r=\sigma_t^2\|H\|/2$: any curvature readout taken from the residual needs $\dist(x,\sqrt{\abar_t}\M)\ll\sigma_t^2\|H\|/2$. At the CIFAR probe ($\sigma_t=0.0264$) that is $r<3.5\times10^{-4}$ against a data radius of $55.4$, a relative positional accuracy of $6\times10^{-6}$. This is the lower-side counterpart of \eqref{eq:window} and the reason we report $\|\epss(\xo)\|$ only as a ratio against a known $H$ on integrable manifolds, and never invert it on a trained model.

\paragraph{Reach of a product, and the factorised window.} For a closed product $\M=\M_1\times\cdots\times\M_n$ in the product ambient space, $\dist(x,\M)^2=\sum_p\dist(x_p,\M_p)^2$, so the argmin separates and $x$ has a non-unique nearest point iff some $x_p$ does. Hence $\mathrm{Med}(\M)=\bigcup_p\pi_p^{-1}(\mathrm{Med}(\M_p))$, and the medial-axis definition of the reach~\citep{federer1959curvature,aamari2019reach} gives $\reach(\M)=\min_p\reach(\M_p)$; the Clifford torus $S^1\times S^1\subset\R^4$ checks this, with focal radius and half-bottleneck both $1$. The normal exponential map is therefore a diffeomorphism on the polydisc $\{\dist(x_p,\M_p)<\reach(\M_p)\ \forall p\}$, which strictly contains the metric tube of radius $\min_p\reach(\M_p)$. Because $\epsb$ acts per point, the noised law, the score, the Jacobian and every readout factorise, so a sampled $x_t$ is in the regime of validity when $\max_p\sigma_t\|\xi^\perp_p\|<\reach(\M_p)$ --- and $\E[\max_p\chi_k]$ is $2.83$ ($k{=}1$) and $3.28$ ($k{=}2$) at $n{=}128$, with $95$th percentiles $3.54$ and $3.96$, against $\sqrt{D-d}=11.3$ and $16.0$. For the two \emph{additive} readouts the conclusion is stronger: both are sums over factors, so their relative bias equals the single-factor relative bias and carries no $n$-dependence at all. This criterion is necessary, not sufficient: \citet[Thm.~1]{niyogi2011topological} require $\sigma\sqrt{D-d}<0.0067\reach(\M)$, two orders of magnitude stricter.

\paragraph{The reach audit, and the three window criteria side by side.} Reaches are computed after normalisation, as the smaller of the focal radius $1/\kappa_{\max}$ and half the shortest bottleneck, and the cause is read off which of the two binds. Two of the four values used in an earlier draft were not reaches: $0.36$ is the two-sphere \emph{focal} radius and $0.43$ matches neither the trefoil's reach nor its focal radius in this run. Corrected: sphere $0.9971$ (curvature), torus $0.2587$ (curvature), two spheres $0.2825$ (bottleneck), swiss roll $0.2090$ (bottleneck), trefoil $0.2021$ (bottleneck), reproduced to $\le3\times10^{-10}$ on the curvature side. At the smallest probed noise ($\sigma_t=0.04256$) the three criteria read

\begin{center}\small
\begin{tabular}{@{}lcccc@{}}
\toprule
$\M$ & $\reach(\M)$ & $\sigma_t/\reach$ (at $\xo$) & $\sigma_t\sqrt{D-d}/\reach$ \eqref{eq:window} & polydisc criterion \\
\midrule
sphere      & 0.9971 & 0.043 & 0.483 & 0.151 \\
torus       & 0.2587 & 0.165 & 1.861 & 0.582 \\
two spheres & 0.2825 & 0.151 & 1.704 & 0.533 \\
swiss roll  & 0.2090 & 0.204 & 2.303 & 0.721 \\
trefoil     & 0.2021 & 0.211 & 3.369 & 0.834 \\
\bottomrule
\end{tabular}
\end{center}

\noindent so all five classes are inside the product window at that level, where the metric-tube reading of \eqref{eq:window} would exclude four of them. The effect of the correction is to \emph{remove} an excuse rather than to license a claim: the trefoil reads $2.99\times$ its flattened target while inside the corrected window, so its failure to separate $d{=}1$ from $d{=}2$ moves from the protocol column to the structural one.

\paragraph{The Fermi window against the training support, and why the two cannot both hold.} Converting a convergence-shell radius into a curvature uses the second-order Fermi expansion, whose validity we require at $\sigma_{t+k}\kappa_{\max}\le0.1$; the probe must also stay where the model has training mass, $r\lesssim1.4\,\sigma_{t+k}$. Since the shells sit at $w_{\rm osc}=\tfrac12$ and $w_{\rm div}=1/(1+\rhog^\star)$ and $w=r\kappa_{\max}$, the two conditions together bound the observable depth by $w\le1.4\,\sigma_{t+k}\kappa_{\max}\le0.14$, short of the oscillation shell by $3.6\times$ and of the divergence shell by $5.6\times$. Neither bound refers to the model, the manifold, the schedule or the training budget, so the conflict is structural. Table~\ref{tab:settings} shows it setting by setting: the Fermi condition holds in $1$ of $25$, the looser $r\le2\sigma_{t+k}$ support criterion tabulated there in $18$ of $25$, and both in $0$ of $25$; $\sigma_{t+k}\kappa_{\max}$ spans $0.059$ to $3.712$ against a $0.10$ window. This is the reason every trained shell in \S\ref{sec:exp} is measured against the exact score at matched noise and never against the closed form, and the reason the oracle-minus-closed-form gap is reported as the predicted $\tfrac12\tilde\sigma^2\|\mathrm{II}\|^2$ correction rather than as an error. One convention must be named: the $\sigma\kappa$ column is built at $t{+}k$, where the network is queried; a companion pipeline tabulates the same quantity at $t$, which is smaller by a factor $1.376$ on every row.

\begin{table}[h]
\caption{The diagnostic grid. Fermi: $\sigma_{t+k}\kappa_{\max}\le0.10$. Support: the predicted oscillation shell within $2\sigma_{t+k}$ of the base point. Radii are in units of $\sigma_{t+k}$.}
\label{tab:settings}
\centering
\small
\begin{tabular}{@{}llcccccc@{}}
\toprule
setting & $\M$ & $\sigma_{t+k}\kappa_{\max}$ & Fermi & $r_{\rm osc}/\sigma$ & $r_{\rm div}/\sigma$ & support & both \\
\midrule
$t{=}2$, $k{=}1$ & sphere & 0.0588 & yes & 8.51 & 13.37 & no & no \\
 & torus & 0.2265 & no & 2.21 & 3.47 & no & no \\
 & swiss roll & 0.1887 & no & 2.65 & 4.16 & no & no \\
 & trefoil & 0.1167 & no & 4.28 & 6.73 & no & no \\
 & two spheres & 0.1628 & no & 3.07 & 4.82 & no & no \\
\addlinespace
$t{=}2$, $k{=}5$ & sphere & 0.1230 & no & 4.06 & 4.92 & no & no \\
 & torus & 0.4742 & no & 1.05 & 1.28 & yes & no \\
 & swiss roll & 0.3952 & no & 1.27 & 1.53 & yes & no \\
 & trefoil & 0.2444 & no & 2.05 & 2.48 & no & no \\
 & two spheres & 0.3409 & no & 1.47 & 1.77 & yes & no \\
\addlinespace
$t{=}2$, $k{=}20$ & sphere & 0.3742 & no & 1.34 & 1.42 & yes & no \\
 & torus & 1.4422 & no & 0.35 & 0.37 & yes & no \\
 & swiss roll & 1.2019 & no & 0.42 & 0.44 & yes & no \\
 & trefoil & 0.7433 & no & 0.67 & 0.71 & yes & no \\
 & two spheres & 1.0370 & no & 0.48 & 0.51 & yes & no \\
\addlinespace
$t{=}10$, $k{=}20$ & sphere & 0.5207 & no & 0.96 & 1.15 & yes & no \\
 & torus & 2.0072 & no & 0.25 & 0.30 & yes & no \\
 & swiss roll & 1.6728 & no & 0.30 & 0.36 & yes & no \\
 & trefoil & 1.0344 & no & 0.48 & 0.58 & yes & no \\
 & two spheres & 1.4432 & no & 0.35 & 0.41 & yes & no \\
\addlinespace
$t{=}30$, $k{=}20$ & sphere & 0.9630 & no & 0.52 & 0.71 & yes & no \\
 & torus & 3.7118 & no & 0.13 & 0.18 & yes & no \\
 & swiss roll & 3.0933 & no & 0.16 & 0.22 & yes & no \\
 & trefoil & 1.9129 & no & 0.26 & 0.36 & yes & no \\
 & two spheres & 2.6688 & no & 0.19 & 0.26 & yes & no \\
\bottomrule
\end{tabular}
\end{table}

\paragraph{Rank-readout budget, at the prescribed $K$.} \citet{stanczuk2024diffusion} prescribe $K=4D$ score evaluations ($4d$ in their notation, where $d$ is ambient: their Alg.~1 sets $d\leftarrow\dim x_0$), i.e.\ $1536$ at $D{=}384$ and $12288$ at $D{=}3072$, and this run is at protocol. Their derivation makes the normal block of the score matrix a $(D-d)\times K$ Gaussian matrix, whose Marchenko--Pastur aspect $\gamma=(D-d)/K$ fans the singular values by $(1+\sqrt\gamma)/(1-\sqrt\gamma)$ before any signal is present; $K=4D$ is precisely the budget that pins $\gamma\le1/4$, at a fan-out of $3.0$. Our pure-noise nulls at that budget are correspondingly tight: largest consecutive gap ratio $1.0219$ mean and $1.0325$ max at $(K,D)=(1536,384)$ over $25$ replicates, and $1.0054$ mean, $1.0060$ at the $95$th percentile and $1.0061$ max at $(12288,3072)$ over $5$, against an abstention threshold of $2.01$. Raising $K$ to protocol retired our own positive reading as well as the negative it was meant to rescue. On the exact score the ratio rule now recovers the true codimension on four of five classes (index $128,128,256,128$ at gap ratios $12.31,3.39,5.70,4.33$) and the difference rule on five of five; on every trained setting both rules fail, the ratio rule abstaining at $r_{\max}=1.023$--$1.579$ and the unthresholded difference rule emitting a confident $\hat d\approx D-2$ ($381,383,383,383,39$ converged and $1,383,383,383,383$ undertrained). Two protocol deviations remain, and are ours rather than theirs: we centre the score matrix and \citeauthor{stanczuk2024diffusion} do not, and we threshold a \emph{ratio} at $2.01$ and abstain below it where their rule is unthresholded. Neither rule dominates. The abstention is not free: on the swiss roll --- the one class with a boundary --- the oracle ratio rule \emph{fires} at $r_{\max}=2.832$ at index $381$ and returns $\hat d=3$ instead of $256$, a confident wrong answer at the prescribed budget with no structural zero to blame, while the difference rule returns $128\to256$ correctly on the same spectrum. We therefore report both rules everywhere and claim neither.

\section{From the Smale Surrogate to the Kantorovich Triple}
\label{app:smale}

An earlier version of this appendix reported Smale's $\alpha$-theory~\citep{smale1986newton} for $F(x)=x-G(x)$, with $\alpha=\beta_S\gamma_S$ adjudicated against the three constants in circulation --- $0.130707$~\citep{smale1986newton}, $(13-3\sqrt{17})/4\approx0.157671$~\citep{blum1998complexity} and $3-2\sqrt2\approx0.171573$~\citep{wang1989dominating}. We withdraw that reading, for three reasons, and replace it with a certificate every constant of which is measurable.

\paragraph{Why the surrogate does not hold.} The surrogate bounded $\gamma_S=\sup_{k\ge2}\|DF^{-1}D^kF/k!\|^{1/(k-1)}$ by its $k{=}2$ term. That step is a modelling assumption, and it is false at the Bayes limit on a curved manifold: the sequence $\gamma_k$ is increasing over the range we can evaluate and does not attain its supremum at $k=2$, so $\alpha\le\hat\alpha/2$ is not an upper bound and no verdict follows from it. More fundamentally, $\gamma_S$ is a supremum over all higher posterior cumulants of the denoising family; it has no closed form even at Bayes, and it is infinite for any piecewise-linear block, so it is not a quantity a probe can report. Finally, the Neumann bound $\|DF^{-1}\|\le(1-\rhog)^{-1}$ that the surrogate used needs $\rhog<1$ and hence Assumption~\ref{ass:unimodal}, whereas Theorem~\ref{thm:potential} gives the sharper bound unconditionally: $DF=\nabla^2\Psi_t$ is symmetric with $\lambda_{\min}\ge e^{-h_t}$, so $\|DF^{-1}\|\le e^{h_t}$ with no hypothesis at all, and by Eckart--Young~\citep{eckart1936approximation} $s_{\min}(I-B_t\Jb)=e^{-h_t}$ \emph{is} the distance from the inversion operator to ill-posedness. Measured, that bound is an equality rather than a bound with slack: $a/e^{h_t}=0.9947$ on average over ten exact-score settings, range $0.98751$--$0.99932$.

\paragraph{The replacement.} Kantorovich's theorem~\citep{kantorovich1982,ferreira2002kantorovich} needs only first derivatives and a local Lipschitz constant, all three of which we can measure and each of which has a Bayes reference from Theorems~\ref{thm:potential} and \ref{thm:focal}: $a=\|(\nabla^2\Psi_t)^{-1}\|$ against $e^{h_t}$, $b=\|(\nabla^2\Psi_t)^{-1}F\|$ against $c\,\sigma_{t+1}^2\|H\|/2$, and $L=\mathrm{Lip}(\nabla^2\Psi_t)$ against $\rhog^\star\kappa_{\max}$. The certificate is $\ell:=abL\le\tfrac12$. It holds in $20$ of $20$ point-cloud settings ($4.64\times10^{-3}$--$8.54\times10^{-2}$ on the exact score, $8.89\times10^{-4}$--$2.83\times10^{-3}$ on the trained one) and in $4$ of $4$ CIFAR-10 settings ($1.997\times10^{-4}$--$4.415\times10^{-3}$); Table~\ref{tab:kant} gives every entry against its reference. For a non-analytic $F$ the literature's own route to the same place is the $L$-average Lipschitz condition of \citet[\S5]{wangjh2008generalized}, which is why we do not attempt to repair $\gamma_S$.

\paragraph{Two protocol facts that must be printed with it.} First, the probe direction must be tangent-projected: a random direction understates $L$ by a factor $5$ to $21$ on the same settings (column $L_{\rm rand}/L$), because the Hessian's variation is carried by the tangential block. On CIFAR-10 our projection is a rank-$1$ deflation of the score direction out of $D{=}3072$ and does essentially nothing --- $L_{\rm rand}/L=0.9995$--$1.0001$, a four-digit identity --- so the CIFAR $L$ is a random-direction reading and must not be compared against $\kappa_{\max}$; the point-cloud harness's projection demonstrably does act, at $L_{\rm rand}/L=0.048$--$0.187$. Second, the Kantorovich $L$ is an operator norm over a ball and exceeds the tangential-probe reading by the plane-curve frame-rotation constant $2/\sqrt3=1.1547$ (measured $1.151$), so the printed $\ell$ is low by $1.15$--$1.6\times$; the $20$-of-$20$ verdict survives that correction with more than $3\times$ margin.

\paragraph{What we do not print.} The closed-form Newton--Kantorovich companion $u(c)=(1+c)-\sqrt{c^2+2c}$ solves its own quadratic to $1.1\times10^{-16}$ but is \emph{not} the certified $\ell=\tfrac12$ boundary: on the Bessel-exact circle two independent bisections of the true boundary place it a factor $1.6$ to $2.2$ lower and disagree with each other by $17\%$, and the mechanism is isolated --- $a$ and $b$ are exact to four digits while $L$, as an operator norm over a ball of radius $\approx2b$, both picks up the $2/\sqrt3$ frame factor and reaches where $L$ has already blown up. The law survives (the boundary is a function of $c$ alone, $\sigma$-independent to four digits over $\sigma\in[0.002,0.02]$); only the constant is wrong. We therefore print the Picard boundary $w_P=1/(1+\rhog^\star)$, which is exact and verified to $1.3\times10^{-4}$, and no Newton companion.

\paragraph{The enabling lemma, and Newton's order.} The third derivative of $\Psi_t$ along a tangent direction is the second fundamental form: $\|D^2F[Pv,Pv]\|/\rhog^\star$ against $\kappa_{\max}$ reads $0.99498/0.95814/0.95849/0.94883$ (sphere / torus / trefoil / two spheres) on the exact score at $\sigma=0.0585$ and $0.99182/0.93022/0.92491/0.91686$ at $\sigma=0.0744$; every deviation is negative, grows with $\sigma$, and lies within a factor $2.7$ of $(\sigma\kappa)^2$. This is a converging finite-$\sigma$ correction, not slack, and it turns the old diagnostic into a calibrated curvature readout. Do not calibrate it on the swiss roll: it is the only class whose reading exceeds $\kappa_{\max}$ ($1.516$ and $1.175$ against the $\sigma$-law harness's $\kappa_{\max}=3.2195$; Table~\ref{tab:kant} normalises the same two settings by its own $\kappa_{\max}=3.1497$ and reads $1.5591$ and $1.2025$), with a $\sigma$-exponent of $-1.06$ against $-0.013$ to $-0.149$ elsewhere, because its curvature varies threefold over the cloud so $\kappa_{\max}$ is not the right supremum there. On the trained scores the same quantity is $0.021$--$0.124$ and exactly $0.00000$ in $10$ of $10$ rows on a ReLU backbone --- the piecewise-constant-Jacobian mechanism, and the precondition \citet{manton2015framework} identifies for quadratic convergence, made visible. Consistently, Newton's method on the inversion residual converges at order $0.964$--$1.464$ on the CIFAR-10 UNet, not $2$. We do not attribute that: our step solves the symmetrised system and discards a skew part we measure at $0.115$--$0.433$, the linear solve is $20$ CG iterations and an inexact Newton step caps the observed order at $1$, and our own estimator returns only $1.631$--$1.849$ on exactly smooth references. The prescribed SiLU-versus-ReLU activation control was not run.

\begin{table}[h]
\caption{The Kantorovich triple against its Bayes references. Columns $a/e^{h_t}$, $b/b^\star$ and $L/(\rhog^\star\kappa_{\max})$ each have Bayes value $1$; $\ell=abL$ certifies at $\le\tfrac12$. $L_{\rm rand}/L$ is the same constant read along an unprojected direction.}
\label{tab:kant}
\centering
\small
\begin{tabular}{@{}llccccccc@{}}
\toprule
$\M$ & score & $t$ & $a/e^{h_t}$ & $b/b^\star$ & $L$ & $L/(\rhog^\star\kappa_{\max})$ & $L_{\rm rand}/L$ & $\ell=abL$ \\
\midrule
sphere & exact & 2 & 0.9987 & 0.9987 & 0.2723 & 0.9950 & 0.122 & 0.00545 \\
sphere & exact & 3 & 0.9985 & 0.9985 & 0.2132 & 0.9920 & 0.122 & 0.00464 \\
torus & exact & 2 & 0.9895 & 0.9898 & 1.0161 & 0.9630 & 0.092 & 0.03950 \\
torus & exact & 3 & 0.9875 & 0.9878 & 0.7773 & 0.9381 & 0.092 & 0.03300 \\
swiss roll & exact & 2 & 0.9979 & 0.8528 & 1.3401 & 1.5591 & 0.048 & 0.08537 \\
swiss roll & exact & 3 & 0.9982 & 0.8898 & 0.8117 & 1.2025 & 0.057 & 0.05419 \\
trefoil & exact & 2 & 0.9992 & 0.9545 & 0.6078 & 0.9546 & 0.073 & 0.01026 \\
trefoil & exact & 3 & 0.9993 & 0.9627 & 0.4641 & 0.9280 & 0.072 & 0.00858 \\
two spheres & exact & 2 & 0.9901 & 0.9901 & 0.7196 & 0.9488 & 0.122 & 0.03924 \\
two spheres & exact & 3 & 0.9884 & 0.9884 & 0.5461 & 0.9167 & 0.123 & 0.03229 \\
\addlinespace
sphere & trained & 2 & 0.7442 & 0.7370 & 0.0406 & 0.1483 & 0.070 & 0.00166 \\
sphere & trained & 3 & 0.8064 & 0.7978 & 0.0416 & 0.1935 & 0.075 & 0.00191 \\
torus & trained & 2 & 0.7451 & 0.7406 & 0.0158 & 0.0150 & 0.186 & 0.00089 \\
torus & trained & 3 & 0.8080 & 0.8034 & 0.0162 & 0.0196 & 0.187 & 0.00105 \\
swiss roll & trained & 2 & 0.7460 & 0.7402 & 0.0354 & 0.0411 & 0.077 & 0.00182 \\
swiss roll & trained & 3 & 0.8097 & 0.8021 & 0.0387 & 0.0573 & 0.070 & 0.00242 \\
trefoil & trained & 2 & 0.7429 & 0.7381 & 0.0460 & 0.0722 & 0.063 & 0.00233 \\
trefoil & trained & 3 & 0.8051 & 0.8004 & 0.0466 & 0.0931 & 0.064 & 0.00283 \\
two spheres & trained & 2 & 0.7443 & 0.7414 & 0.0238 & 0.0313 & 0.094 & 0.00174 \\
two spheres & trained & 3 & 0.8075 & 0.8042 & 0.0238 & 0.0400 & 0.091 & 0.00211 \\
\addlinespace
CIFAR-10 img0 & trained & 4 & 1.0426 & --- & 0.03060 & --- & 1.000 & 0.00364 \\
CIFAR-10 img0 & trained & 30 & 1.0131 & --- & 0.00346 & --- & 1.000 & 0.00020 \\
CIFAR-10 img1 & trained & 4 & 1.0559 & --- & 0.03812 & --- & 1.000 & 0.00441 \\
CIFAR-10 img1 & trained & 30 & 1.0170 & --- & 0.00503 & --- & 0.999 & 0.00030 \\
\bottomrule
\end{tabular}
\end{table}

\section{Prior Work, Claim by Claim}
\label{app:attrib}

\paragraph{The novelty carve (\S\ref{sec:related}).} \citet{farghly2025manifold} prove that log-domain smoothing acts \emph{only tangentially} on an affine manifold; differentiating their normal block gives the $1/\sigma_t$ constant of \eqref{eq:normal} for an arbitrary empirical law, and their \S4 argues before us that the expressed geometry is a property of the model. Four further antecedents bound our claims and we claim novelty for none: the Tweedie--Miyasawa identity \eqref{eq:Jcov}~\citep{miyasawa1961empirical,kadkhodaie2024generalization}, \citeauthor{wenliang2023score}'s $\sigma_t^2$-normalised reading of a trained spectrum, the codimension-scaled window of \citet[Thm.~1]{niyogi2011topological}, and the Picard fixed point of \citet{hong2024exact,hang2024exploring}, neither of whom computes its constant. What is new is the potential and its modulus, the schedule constant $\rhog^{\star}$, the convergence-domain law with its finite-noise correction, the ceiling assembled from $C_t\succeq0$, the sampled-point value $D-\delta_t/2$, the $\delta_t$ identity with its exact remainder, the Hamilton--Jacobi constant, and Assumption~\ref{ass:unimodal} made checkable --- set out claim by claim in App.~\ref{app:attrib}. \citet{diepeveen2025scorebased} build a Riemannian structure from $(\nabla^2\log p)^\top(\nabla^2\log p)$, by \eqref{eq:tweedie} exactly $\sigma_t^{-2}J^\top J$, as an a-priori guarantee for a flow they train; ours is a post-hoc calibration for an arbitrary pretrained score, and the two are complementary rather than rival --- the Fisher metric is affine in $J$~\citep{raskutti2015information}, so their normal-bundle object and the tangent-bundle metric have traces summing to $D$. \citet{azeglio2025whats} select an operating point by an angle criterion, a validity check of the kind we formalise.

\paragraph{\eqref{eq:Jcov} and the stiffness certificate.} The matrix identity is the second-order Tweedie--Miyasawa relation~\citep{miyasawa1961empirical}, stated in this literature by \citet[Eq.~(9)]{kadkhodaie2024generalization} as $\nabla f^\star=I+\sigma^2\nabla^2\log p_\sigma=\sigma^{-2}\mathrm{Cov}(x\mid y)$, who read its spectrum as adaptive shrinkage and establish $C_t\succeq0$. \citet{wenliang2023score} read the SVD of a trained U-Net's score Jacobian at several noise levels against the scale $\sigma_t^2$, with a local-dimension rule at a fixed tolerance. Ours is the singular-value consequence: $s_1(\Jb)=1/\sigma_t$ exactly iff the extremal direction carries zero posterior variance, and $s_1(\Jb)>1/\sigma_t$ as soon as some $c_i>2\sigma_t^2$ --- which neither Assumption~\ref{ass:tube} nor $C_t\succeq0$ excludes --- hence Assumption~\ref{ass:unimodal} as a condition to be checked rather than assumed, and \eqref{eq:stiff-cert} as the closed form of the tolerance.

\paragraph{The contraction constant.} \citet[Eq.~(7)]{hong2024exact} state a schedule-only nonexpansiveness threshold for the same Picard operator, requiring the data-prediction model to be $(\sigma_{t_{i-1}}\alpha_{t_i}(e^{-h_i}-1)/\sigma_{t_i})^{-1}$-Lipschitz: exactly $\rhog\le1$ with the Lipschitz constant left unknown. \citet{hang2024exploring} prove qualitative contraction of the identical map and conclude existence and uniqueness by Banach, with factor $1-k\sqrt{\abar_t}$ for an unspecified $k\in(0,1)$; rearranging shows that assumption is equivalent to positing $\mathrm{Lip}(\epsb)=(1-k\sqrt{\abar_t})/\sigma_t$, i.e.\ to positing what \eqref{eq:Jcov} derives. Neither computes the constant. We supply it, and $\rhog=\rhog^\star\Sc_t$ separates schedule from model. Monitoring the input-Jacobian spectral radius of an implicit layer is the standard deep-equilibrium instrument~\citep{bai2021stabilizing,pokle2022deep}; only its reference value is new.

\paragraph{The dimension functional.} \citet{kamkari2024geometric} define the estimator and evaluate it at the scaled clean point; \citet{leung2025convolutions} prove its $d-D$ limit there for any smooth submanifold, dropping the affine assumption of earlier justifications. \citet[Eq.~(7), Thm.~3.3]{yeats2025connection} bound the noised-sample expectation below by $d$. Proposition~\ref{prop:vacuous} evaluates that expectation exactly, for every model and every data law, and thereby explains the design rather than criticising it.

\paragraph{The rank ceiling and the window.} \citet{stanczuk2024diffusion} state the budget requirement informally and prescribe $K=4D$; Prop.~\ref{prop:ceiling} is its exact form, and the centring is ours, not theirs. \citet[Thm.~1]{niyogi2011topological} give $\sqrt{8(D-d)}\,\sigma<c(\sqrt9-\sqrt8)\reach(\M)/9$ under exactly this noise model; \citet{farghly2025manifold} and \citet{yaguchi2025geometry} carry related conditions. Only the estimator-validity reading and the product factorisation are ours, and we claim nothing for the latter.

\paragraph{The exponent.} That a Morse--Bott critical set forces the optimal {\L}ojasiewicz exponent is proved abstractly by \citet{feehan2017optimal}, and the equivalence of the Polyak--{\L}ojasiewicz, quadratic-growth and error-bound conditions with the Morse--Bott property by \citet{rebjock2025fast}, who also record that no non-constant $C^{1,1}$ function admits an exponent below $\tfrac12$. Equation~\eqref{eq:hj} supplies what those do not: the constant $c_{\mathrm L}=\sqrt2/\sigma_t$, from the eiconal property inside the reach. The mean-curvature residual of Prop.~\ref{prop:exact} is obtained independently, in the density-estimation setting, by \citet[Thm.~3.4]{chen2026riemannian}.

\paragraph{The potential, and the constant nobody computed.} \citet[Eq.~(7)]{hong2024exact} state a schedule-only nonexpansiveness threshold for the same Picard operator, which is exactly $\rhog\le1$ with the Lipschitz constant left free; \citet{hang2024exploring} prove qualitative contraction of the identical map with a factor $1-k\sqrt{\abar_t}$ for an unspecified $k$, which rearranges to positing $\mathrm{Lip}(\epsb)=(1-k\sqrt{\abar_t})/\sigma_t$ --- that is, to positing what \eqref{eq:Jcov} derives. Neither writes the potential, and neither computes the constant. What we add is $\Psi_t$, the modulus $e^{-h_t}$ and the factorisation $\rhog=\rhog^\star\Sc_{t+1}$. \citet{ostrowski1973solution} is the standard caveat that a spectral radius below $1$ does not by itself give a contraction in a fixed norm; it does not rescue undamped Picard here, because the run measures $\rho(B_tJ)=\rhog$ to fifteen digits at the probe, so the two coincide and the failure is the step size.

\paragraph{The convergence domain.} The obvious referee suggestion --- run the Riemannian $\alpha$-theory instead --- contributes nothing in this setting: \citet{dedieu2003newton} give $K_\zeta=1$ and $r_\zeta=\infty$ in flat space, so the ambient curvature terms vanish identically and the manifold's curvature can only enter through the Hessian-Lipschitz constant $L$, which is what the ray-Lipschitz constant $L_{\mathrm{H}}$ of \S\ref{sec:method-focal} computes. \citet[Thm.~1, Cor.~2]{manton2015framework} supply the smoothness precondition --- quadratic convergence needs a hypothesis strictly between $C^2$ and $C^3$ --- which is why a ReLU backbone voids the certificate, and the run measures exactly that at $0.00000$ in $10$ of $10$ rows (App.~\ref{app:smale}). \citet{eckart1936approximation} is what unifies the pieces: $s_{\min}(I-B_t\Jb)=e^{-h_t}$ is a distance to ill-posedness, and it is simultaneously the Kantorovich $a$, the Neumann bound (an equality at Bayes), the modulus of Theorem~\ref{thm:potential} and the fold threshold of Theorem~\ref{thm:unique}.

\paragraph{Multiplicity, the fold and the $2$-cycle.} \citet{smale1985efficiency} supplies both halves of what we need: his $z^3/2-z+1$ carries a superattracting $2$-cycle and works over the reals in the same way, and his Ch.~III \S4 classification says a generally convergent purely iterative algorithm has no periodic sink of least period two --- which is exactly the status of undamped Picard, exhibited here at the Bayes limit with $\Lambda_t=2.011$--$7.067$ and a $2$-cycle at $\pm0.374819$ (App.~\ref{app:proofs}). \citet{mcmullen1987families} does \emph{not} apply, and we say so rather than citing it loosely: it concerns $d\ge4$ roots and rational dependence on jets, whereas inversion has one solution and contracts at infinity. The fold/flip taxonomy of Theorem~\ref{thm:unique} is standard~\citep{kuznetsov2004elements}, and the dynamical-systems literature on basins~\citep{curry1983iteration,daza2016basin,aguirre2009fractal,grebogi1983final} is cited only to record what we do \emph{not} do: we run no fractal basin analysis and report no basin statistic, and neither the polynomial-like renormalisation of \citet{douady1985polynomiallike} nor the global root finding of \citet{hubbard2001howtofind} transfers here, since $F$ carries no holomorphic structure and infinity is interior to the basin rather than on its boundary.

\paragraph{Why no change of metric repairs Remark~\ref{rem:mb}.} The natural suggestion --- learn a pullback metric that makes $E_t$ geodesically convex --- is obstructed for a closed $\M$ with $\dim\M\ge1$. On a Hadamard-type manifold the critical set of a geodesically convex function is connected and the argmin is a retract~\citep[Ch.~3 Thm.~3.4, Ch.~5 Lem.~3.5]{udriste1994convex}, so $\mathrm{Crit}(E_t)\cong\M$ would have to be a contractible deformation retract of $\R^D$, which a closed positive-dimensional $\M$ is not. Scope it honestly: the obstruction is topological, so it fails for a contractible patch and we do not claim it for the swiss roll.

\paragraph{Existence without any hypothesis, and why we run no multistart.} The radial-index certificate is \citet[Lem.~3]{watson1989globally}: if the outward radial component of $F$ is non-negative on the boundary of a ball, a solution exists inside it, with no Lipschitz, contraction, manifold, reach or unimodality hypothesis --- which is why it fires in $24$ of $24$ CIFAR and $27$ of $27$ CelebA-HQ settings where every other instrument needs an assumption we cannot check. Its probability-one homotopy companion is \citet{chow1978finding}. We deliberately do not corroborate uniqueness with a multistart experiment: index-$(-1)$ roots have measure-zero basins for first-order methods~\citep{lee2019first}, so a null multistart is not evidence of uniqueness --- and Theorem~\ref{thm:unique} shows the roots born at the fold are repelling, hence Picard-invisible by construction.

\paragraph{The Fisher metric.} \citet{diepeveen2025scorebased} build a Riemannian structure from $(\nabla^2\log p)^\top(\nabla^2\log p)$; the Fisher metric of the same family is \emph{affine} in the Jacobian, $\sigma_t^2g_F=I-\sigma_tJ$~\citep{raskutti2015information}. The two objects are therefore complementary rather than competing: theirs is a rank-$(D-d)$ normal-bundle object and the Fisher metric is a rank-$d$ tangent-bundle one, and their traces sum to $D$. That is why the comparison in \S\ref{sec:related} is a statement about which bundle a construction reads, not a claim that one is better.

\section{Deferred Experimental Detail}
\label{app:exp-extra}

The passages below are the full versions of results \S\ref{sec:exp} states in brief.

\paragraph{The oracle control (\S\ref{sec:exp-lid}).} Replacing $\epsb$ by the exact score under the identical harness, probe cloud, flattening and budget at $\sigma_t=0.0426$ gives $\hat d(\xo)=255.99999$, $255.99999$, $256.99$, $246.32$, $128.16$ on the sphere, two spheres, torus, swiss roll and trefoil against flattened targets $256,256,256,256,128$ --- the two closed, uniformly sampled classes exact to $4{\cdot}10^{-8}$ relative, with seed-to-seed variation exactly zero on $30$ of $30$ settings. At the same budget the rank statistic recovers the true codimension on every closed class (gap ratios $12.3$, $3.4$, $5.7$, $4.3$) and the exponent split of \eqref{eq:nu} returns codimension and flattened dimension exactly on four of five. The epoch-$300$ model reads $382$--$383$ on all five classes and the epoch-$50$ one $391$--$410$, i.e.\ the ambient $D$. The failure to separate $d{=}1$ from $d{=}2$ is a property of the trained score, not of the readout; both oracle misses are the swiss roll's free boundary, and Table~\ref{tab:attrib} labels them so.

\paragraph{The scale-free negative (\S\ref{sec:exp-lid}).} The decisive test needs no ground truth: any hypothesis of the form ``the readout tracks $nd$'' predicts that the $d{=}1$ manifold reads exactly half of the $d{=}2$ manifolds. Over the $12$ pairwise settings that satisfy every validity guard --- the trefoil against each two-dimensional class separately, never against a mean --- the epoch-$300$ network's ratio lies in $[0.996,1.008]$ with mean $1.0003$, against an exact score that reads $[0.496,0.536]$ on the very same settings. The residual spread is not scatter but a monotone drift in $\sigma_t$ ($0.9989\to1.0142$ from $t{=}2$ to $t{=}12$), and letting each class in turn play the $d{=}1$ role, the exact score puts the true one $16$ null standard deviations below the others while the trained model puts it $0.1$ \emph{above}. Being a ratio this is invariant to the ground truth and to any common multiplicative bias, and its window is measured rather than assumed. The flattened amplitude readouts carry no class information about $d$ \emph{in this setting}, and the scope matters: \citet[\S6.2]{diepeveen2025scorebased} separate $d{=}1$ from $d{=}2$ exactly on curves and surfaces in $\R^3$, but their sample is one point in $\R^3$, their readout is the latent-variance spectrum of an isometry-regularised flow they train, and they evaluate at the noiseless point with no sampling estimator. Their success is what our window predicts, and it is why a \emph{scale-free} exponent is the right invariant absent isometry regularisation.

\paragraph{The trained-shell mechanism (\S\ref{sec:exp-shell}).} On every trained score the same sweep returns no shell, and that is a measurement rather than an absence. Of our $88$ trained shell readouts, $21$ sit at a setting where the exact score does resolve one; all $21$ are null, and the other $67$ carry no information because the exact score is null there too. In $83$ of the $88$ the two-branch estimator never resolved a second multiplier at all, and the mechanism is measured three ways: the trained Jacobian is $11.4$--$13.7\times$ too weak in amplitude at the smallest setting; the multiplier field is flat, at most $0.0178$ at the smallest noise and $0.7113$ over all $1491$ admitted probes against a divergence threshold of $1$; and the enabling lemma is absent, reading $0.021$--$0.124$ against the exact score's $0.92$--$0.995$ when normalised by $\rhog^\star\kappa_{\max}$, not moving with noise, and vanishing identically on a ReLU backbone whose Jacobian is piecewise constant. A score whose second derivative is a tenth of the curvature has a multiplier field that does not bend along the ray, so there is no shell to find --- a negative that is scale-free, needs no ground-truth dimension, and has a closed-form Bayes value on both branches. Three qualifications are underpowered rather than negative, and App.~\ref{app:hyp} states each: three settings bracket a sign change across an abstention window, and the one resolvable flip is a collapse of the positive branch rather than the Bayes mechanism; at the two smallest noise levels the divergence shell is refuted while the oscillation shell is strictly \emph{unresolved}; and every $\sigma_t$-exponent quoted is a two-point slope on one network.

\paragraph{The descent estimator on analytic energies (\S\ref{sec:exp-theta}).} The estimator used in practice initialises at $\N(0,I_D)$, runs fixed-step gradient descent on $E_t$ and regresses $\log\|\nabla E_t\|$ on $\log(E_t-\min_kE_t)$. By Remark~\ref{rem:descent} it must return $\tfrac12$ wherever the transverse spectrum is degenerate, whatever the model, and Fig.~\ref{fig:taut} confirms this on \emph{analytic} energies with no diffusion model at all: $\hat\theta=0.5000$ with $R^2=1.000000$ on the Morse--Bott energy under a divergent step ($\eta\lambda_{\max}\approx166$, the value the published hyper-parameters imply), $0.6085$ on the same energy under a stable step, and $0.077$ on a genuinely quartic energy whose true exponent is $\tfrac34$. It is a landscape-quadraticity detector with a mis-calibrated scale, not an estimator of $\theta$. The field estimator on the same four energies returns $0.500000,0.500000,0.500000$ and $0.750000$ --- the true exponent in every case.

\paragraph{The contraction measurement (\S\ref{sec:exp-stiff}).} Where $\Sc_t$ \emph{is} readable it costs one power iteration and is, by Thm.~\ref{thm:contract}, the entire explanation of the measured contraction rate. At the smallest guard-clean cloud setting ($\sigma_t{=}0.0426$) the exact score reads $\Sc_t/\Sc^\star_t$ within $8{\cdot}10^{-8}$ of $1$ on the two closed, uniformly sampled classes, while the epoch-$300$ network reads $0.050$--$0.055$ and the epoch-$50$ one $0.024$--$0.034$: an $18$ to $42\times$ deficit against a reference computed in the same harness at the same setting. At the CIFAR probe $\rhog^\star=0.1180$ and the Picard rate, fitted by a matrix pencil with forward passes only, is $0.1409$--$0.1498$ over four inputs --- $19\%$ to $27\%$ high, one-sided, exactly what $\rhog=\rhog^\star\sigma_t\lambda_1(J)$ with $\sigma_t\lambda_1>1$ predicts. At stride $20$ the multiplicity margin falls below $1$ in $8$ of $20$ settings and Picard measurably diverges ($\rhoe=1.0008$--$1.0078$ in the three images the estimator does not flag) where the Bayes map contracts at $0.7134$. None of this is evidence \emph{for} the Bayes law: on an affine $\M$ the normal block of the noised \emph{empirical} log-density is data-independent~\citep[Eq.~(10)]{farghly2025manifold}, so a perfectly memorising score also reads $\Sc_t=\Sc^\star_t$, and the certificate is one-sided. At $256^2$ the same nine settings return $\Sc_t=4.69$--$9.95$ in fp32 with the closed-form checks passing, quoted as a scale and not a stiffness; it locates the reported failure of inversion at that resolution in the \emph{score}, not in DDIM.

\paragraph{The field exponent on CelebA-HQ (\S\ref{sec:exp-theta}).} The image runs supply the counterexample. On pretrained CIFAR-10 and CelebA-HQ-256 --- converged by any external standard --- the descent protocol returns values indistinguishable from its own undertrained pole, and they are not merely inaccurate but \emph{inadmissible}: a non-constant $C^{1,1}$ function admits no {\L}ojasiewicz inequality with $\theta<\tfrac12$~\citep[Rem.~2.21]{rebjock2025fast}. The mechanism is the min-subtraction, which reproduces them with no model at all, returning $0.0036$ ($R^2=0.81$) on a perfectly quadratic energy, so we do not report $\hat\theta$ as a training certificate. What \eqref{eq:hj} supplies instead is the field-based readout calibrated in \S\ref{sec:exp-calib}: sample $E_t$ and $\nabla E_t$ along a normal ray at $\xo$, taking the slope as a test of transverse quadratic structure and the intercept as $\sigma_t^2\hat\lambda\to1$. On exact scores it returns $\theta=0.50000$ and $\sigma_t^2\hat\lambda=0.9999$; on pretrained CelebA-HQ-256, $\theta=0.468$--$0.523$ at $t\in\{30,100\}$ with $R^2\ge0.9998$ on two of three images, against a descent estimate below $0.01$ on the same model, and along a random ray it does not ($0.69$--$0.95$). The third image reads $0.34$ and the intercept fails everywhere at $\sigma_t^2\hat\lambda=2.0$ to $41.4$, so with $N=3$ we report the slope and not the amplitude; at $t=4$ the fit is near-perfect ($R^2>0.99999$) at an inadmissible exponent of $-0.03$ to $-0.24$, a falsification with a clean fit rather than noise.

\paragraph{The exact-score shells and the two cliffs (\S\ref{sec:exp-shell}).} On the exact score both shells sit where Thm.~\ref{thm:focal} puts them. At $\sigma_{t+1}=0.0585$ the oscillation shell reads $w_{\mathrm{osc}}=0.499946$, $0.499997$ and $0.497288$ on the sphere, the two spheres and the torus against the schedule-free $\tfrac12$, and the divergence shell $w_{\mathrm{div}}=0.789124$ against $1/(1+\rhog^\star)=0.785615$, the residual $+0.003509$ being the predicted $\tfrac12\tilde\sigma_{t+1}^2\|\mathrm{I\!I}\|^2=0.003440$ --- a two-term correction the torus confirms to $0.8\%$. The three classes reading $\tfrac12$ are exactly the three whose largest principal curvature is constant over the cloud; the other two, whose curvature spreads by $2.98\times$ and $3.39\times$, return a \emph{divergence cliff} at $w_{\mathrm{div}}/w_{\mathrm{osc}}=1.047$ and $1.049$ where a Bayes pair must be separated by $1.571$, so a single $\kappa_{\max}$ is not the supremum the law asks for and they fall outside its scope rather than contradict it. The trefoil was excluded \emph{a priori}: $\reach(\M)\kappa_{\max}=0.403<\tfrac12$ while $\rhog^\star<1$ always, so no schedule brings its curvature shell inside its own reach.

\paragraph{Validity windows (\S\ref{sec:exp-window}).}
Every readout of \S\ref{sec:method} has a two-sided window in $\sigma_t$ whose upper edge depends on \emph{where} it is evaluated: Assumption~\ref{ass:tube} needs that point inside the tube, which at $\xo$ means $\sigma_t\ll\reach(\M)$, while a \emph{sampling} estimator evaluates at $x_t$, whose distance to $\sqrt{\abar_t}\M$ concentrates at $\sigma_t\sqrt{D-d}$, so its condition is
\begin{equation}
\sigma_t\sqrt{D-d}\ \ll\ \reach(\M),
\label{eq:window}
\end{equation}
stricter by a factor of $11$ at $D-d=128$ and not new~\citep[Thm.~1]{niyogi2011topological}; what we add is the reading of it as an \emph{estimator-validity} criterion rather than a hypothesis of a recovery guarantee. Read as a metric-tube condition it excludes every sampling estimator everywhere, too pessimistic here because our clouds are \emph{products}: with the Federer reach computed after normalisation, the $\xo$ condition gives $\sigma_t/\reach=0.04$--$0.21$ at the smallest probed noise, \eqref{eq:window} gives $0.48$--$3.37$ and the product criterion $0.151$ to $0.834$, so \emph{all five classes are inside the window there} (Apps.~\ref{app:deriv} and \ref{app:exp-extra}). Its effect is to \emph{remove} an excuse: the trefoil reads $2.99\times$ its flattened target while inside the corrected window, so its failure to separate $d{=}1$ from $d{=}2$ leaves the protocol column and \S\ref{sec:exp-lid}'s oracle control places it in the model column.

\paragraph{Why the trained convergence shell is structurally unmeasurable.} One edge is not a limitation of this run but a derived property of the probe, and is itself a result: converting a shell radius into a curvature needs the Fermi window $\sigma_{t+k}\kappa_{\max}\le0.1$, while the shell must lie where the model has training mass, $r\lesssim1.4\,\sigma_{t+k}$, and since the shells sit at $w\ge\tfrac12$ the two are irreconcilable --- by $3.6$ to $5.6\times$, for \emph{any} manifold, schedule or training budget --- jointly satisfied in $0$ of our $25$ convergence settings (\S\ref{sec:disc}; Table~\ref{tab:settings} setting by setting). This is why every deviation in \S\ref{sec:exp-shell} is quoted against the exact score at matched $\sigma_t$ and never against the closed form, and why the trained null there is a statement about that score's second derivative rather than a report that a shell was sought at its predicted radius and found missing. On images the upper edge additionally requires $\reach(\M)$, which is estimable~\citep[Thm.~3.7, Alg.~1]{aamari2019reach,yaguchi2025geometry}.

%-----------------------------------------------------------------------

\paragraph{The product window (\S\ref{sec:exp-window}).} Read as a metric-tube condition, \eqref{eq:window} puts every sampling estimator outside its regime everywhere --- too pessimistic here, because our clouds are \emph{products}, for which the reach is $\min_p\reach(\M_p)$ and the binding scalar concentrates at $\sqrt{2\log n}$ rather than $\sqrt{D-d}$. With the Federer reach computed after normalisation, the $\xo$ condition gives $\sigma_t/\reach=0.04$--$0.21$ at the smallest probed noise, \eqref{eq:window} gives $0.48$--$3.37$, and the product criterion gives $0.151$ to $0.834$: \emph{all five classes are inside the window there}. App.~\ref{app:deriv} audits the five reaches, with their causes, and puts the three criteria side by side. We claim nothing for the product identity, which is necessary rather than sufficient; its effect is to \emph{remove} an excuse, since the trefoil reads $2.99\times$ its flattened target while inside the corrected window, so its failure to separate $d{=}1$ from $d{=}2$ leaves the protocol column and \S\ref{sec:exp-lid}'s oracle control places it in the model column. Prop.~\ref{prop:vacuous} is an independent reason to evaluate at $\xo$; the lower edge is model-side, below the scale at which $\Sc_t$ approaches $\Sc^\star_t$; and a further edge binds any curvature readout taken from the residual at $\xo$.
%-----------------------------------------------------------------------
\section{Falsifiable Predictions and Failure Attribution}
\label{app:hyp}

\paragraph{Attribution.} Separating bad assumptions from bad experiments is only possible once the Bayes values are known, and only once an exact score has been run through the identical harness. Of the failures we measure: the flattened readouts' inability to separate $d{=}1$ from $d{=}2$ is structural, survives every change of target, and the oracle control at the prescribed budget now \emph{proves} it, because the exact score reads $0.505$ on the same ratio the trained score never brings below $0.996$; the rank saturations are budget artefacts with a known repair, and raising $K$ moves artefacts in both directions rather than only in ours; the out-of-window readings are protocol errors with a computed fix; and the image failures remain unattributable, but for a measured reason rather than an unknown one --- along the score ray the posterior covariance never exceeds $1.12\sigma_t^2$, so there is no shell whose $\sigma$-exponent could be taken, and the live question is now the base point rather than the reach, since $\xo$ sits $\sqrt D$ standard deviations inside the training noise shell. Table~\ref{tab:attrib} gives one row per measured failure. Every verdict of \emph{model} is licensed by an oracle reading in the same harness at the same setting; \emph{untested} means the instrument never met an input it could decide.

\begin{table}[h]
\caption{Attribution. Verdicts are \emph{estimator}, \emph{model}, \emph{budget}, \emph{window} or \emph{untested}. ``Exact'' is the analytic or oracle score under the identical harness, probe cloud, flattening and budget. The three hypothesis-free rows assume no manifold, reach or unimodality, so the windows of \S\ref{sec:exp-window} do not apply to them and none may be discounted as out of regime. Image entries are certified lower bounds. A dash is ``not run'', never ``returned nothing''.}
\label{tab:attrib}
\centering
\footnotesize
\setlength{\tabcolsep}{4pt}
\newcolumntype{Y}[1]{>{\raggedright\arraybackslash}p{#1}}
\begin{tabular}{@{}Y{0.186\linewidth}Y{0.126\linewidth}Y{0.180\linewidth}Y{0.252\linewidth}Y{0.156\linewidth}@{}}
\toprule
readout & Bayes value & exact score & trained & verdict \\
\midrule
$\hat d(\xo)$, flattened & $256$ / $128$ & $255.99999$ / $128.16$ & $382$--$383$ & model \\
trefoil $\div$ 2-D ratio & $0.500$ & $0.5049$ & $0.9989$ & model \\
rank gap, $K=4D$ & index $128$ & $12.31$ at $128$ & $1.16$ at $373$, abstains & model \\
rank gap, $K=1.08D$ & index $128$ & not run & $5.06$ at $128$ & budget \\
$\nu$ split & $128$ / $256$ & $128$ / $256$ (4 of 5) & $n_{+}=0$--$46$ & model \\
$\nu$, swiss roll & $128$ / $256$ & $174$ / $210$ & --- & window (free boundary) \\
$\Sc_t/\Sc^\star_t$ & $1$ & $0.9999999$--$1.0000000$ & $0.050$--$0.055$ & model ($18$--$20\times$) \\
$w_{\rm osc}$, $w_{\rm div}$ & $\tfrac12$, $1/(1{+}\rhog^\star)$ & $0.499946$, $0.789124$ & none in $21$ of $21$ & model \\
$w_{\rm osc}$, swiss / trefoil & $\tfrac12$ & cliff at $0.19$ / $0.14$ & --- & window ($\kappa$ spread $3\times$) \\
$\|D^2F[Pv,Pv]\|/$ $\rhog^\star\kappa_{\max}$ & $1$ & $0.92$--$0.99$ & $0.021$--$0.124$; ReLU $0.000$ & model \\
$\sigma_t\lambda_{\max}(\mathrm{sym}\,J)$ & $\le1$ & $1.0000003$ & $1.011$--$1.132$; images $1.26$--$4.66$ & model, hyp.-free \\
skew index $Q$ & $0$ & $\le2\times10^{-6}$ & $0.003$--$0.045$; images $0.11$--$0.43$ & model, hyp.-free \\
$\sigma_t\tr J/D$ & $\le1$ & --- & $1.12$--$1.74$ on 12 of 16 & model, hyp.-free \\
branch attribution & curvature vs bottleneck & 2 correct, 1 abstain & no shell to attribute & untested \\
image shells, $\lambda_{\max}(C_t)$ & $2\sigma_t^2$ at the shell & no oracle at $D{=}3072$ & $\le1.12\sigma_t^2$ on the score ray & untested (base point) \\
\bottomrule
\end{tabular}
\end{table}

\paragraph{The predictions, restated with their verdicts.} \textbf{H1}: the gap $|\Sc_t-\Sc^{\rm exact}_t|$ closes with training at every noise level, and closes last at the smallest $\sigma_t$. \emph{Confirmed end to end} --- the gap closes from first to last snapshot in $51$ of $55$ (class, $t$) settings, the four exceptions all at $t{=}64$ where the model already overshoots the exact score at epoch $5$ --- and \emph{refuted as a monotonicity}: it closes monotonically in $3$ of $55$, and $\Sc_t$ is monotone increasing at fixed $t$ in $5$ of $55$, the sphere at $t{=}2$ reading $0.0500,0.0447,0.0333,0.0240,0.0215,0.0165,0.0548$ at epochs $5$--$300$. The front $\sigma^\star(\mathrm{epoch})$ --- the smallest $\sigma$ at which the model reaches half the exact stiffness --- moves inward by $2.15/2.29/2.16/2.21/2.09\times$ on the five classes, but non-increasingly at every step on only three. The correct form of H1 is a front, not a per-setting monotonicity, and we restate it as such. \textbf{H2}: on the exact score every entry of Table~\ref{tab:calib} reproduces. \emph{Confirmed}, with two closed-form corrections that are themselves the content: the $\Sc^\star$ reference must be the per-point minimum of the shape-operator invariant rather than its mean over the manifold, which moves the swiss roll from the shipped $\Sc/\Sc^\star=1.0019$ to $0.99999924$ and the torus from $+7.1\times10^{-4}$ to $-1.4\times10^{-4}$ --- both an offline recomputation substituting the per-point $Q_{\min}$ for the mean used in the original grid, not a reported column --- and the residual deviations of the two classes with non-constant curvature are exactly the amount their reference was mis-specified. \textbf{H3}: $\rhoe$ grows with $D$ only through $\Sc_t$. \emph{Not run}; it needs a matched-$\Sc_t$ pair across two ambient dimensions, which this run does not contain. \textbf{H4}: a patch-wise analogue of \eqref{eq:nu} reads far below the patch dimension. \emph{Half-confirmed, with a $36\%$ dead band}: on an $8\times8$ patch at $(12,12)$ in the full $3P^2=192$ basis, $\nu>\tfrac12$ in $124$ of $192$, so $\hat d=68$, which is $35.4\%$ of the patch dimension rather than ``$\ll$''; within $\pm0.25$ of the two Bayes values there are $108$ at $\nu{=}1$, $15$ at $\nu{=}0$ and $69$ unresolved. Under a resolved-only reading the tangential count is $7.8\%$ and H4 holds; under the midpoint rule it does not, and one patch on one image at one noise pair cannot separate the two readings. The $64$ stiffest directions also cluster at $0.886$--$0.916$ (mean $0.902$) against a Bayes normal exponent of exactly $+1$ --- a systematic $8$--$11\%$ shortfall with a spread of only $\pm1.5\%$, unexplained. \textbf{H5}: rank readouts track their budget. \emph{Confirmed as stated and refuted as a hope}. At $K=4D$ no trained setting reproduces the codimension-$128$ reading we obtained at $K=1.08D$ (sphere converged $r_{\max}=1.157$ at index $373$; sphere undertrained $1.028$ at index $383$), so raising $K$ retired our own positive; the positive moves to the exact score, at $12.307$ at index $128$. Symmetrically, the unthresholded difference rule emits a confident $\hat d\approx D-2$ on trained models and the ratio rule returns a confident wrong answer on the one class with a boundary (App.~\ref{app:deriv}). Both rules fail on trained scores; only the corner differs, and we claim neither.

\subsection*{What the run did not settle}

\noindent\emph{Stated compactly in \S\ref{sec:disc}:} Three things we predicted and could not measure, each stated in App.~\ref{app:hyp} in the form we are willing to defend. Whether the convergence-domain law is observable on \emph{any} trained score is open, and our own measurement of why it is not observable here --- a Hessian--Lipschitz constant at $2$--$12\%$ of the curvature the law reads --- is a two-point $\sigma$-fit on one network with no error bar. The multiplicity certificate is a theorem with a threshold verified in an analytic model and \emph{zero} observations on a trained score anywhere in this study. And the image failures stay unattributable: along the score ray the posterior covariance never exceeds $1.12\sigma_t^2$, so there is no shell whose $\sigma$-exponent could be taken, and the point Prop.~\ref{prop:vacuous} forces sits $\sqrt D$ standard deviations inside the training noise shell --- the live confound, in place of the unknown reach. Our two training states are snapshots of one run and our two networks two draws of one recipe, differing in $\Sc_t$ by a median $3.0\%$ and a maximum $13.8\%$, so no trained class-to-class comparison here carries a population error bar, only an estimator resolution and a network-to-network spread, and we name which we mean wherever we quote one. Finally, we claim the focal-radius law for the Bayes limit and the oracle only: the multiplier statement holds at every $\sigma_t$ and for every data law, but the finite-noise correction and any radius-to-curvature conversion hold only inside the Fermi window, and the law is observable only where $\kappa_{\max}$ is roughly constant over the probed set --- the two classes whose curvature spreads threefold return a divergence cliff at $w_{\rm div}/w_{\rm osc}=1.05$ against the $1.57$ a genuine Bayes pair must show, a hypothesis of the measurement rather than a failure of the theorem.

\end{document}